\documentclass{article}

\usepackage{microtype}
\usepackage{graphicx}
\usepackage{subcaption}
\usepackage{booktabs} 

\usepackage{hyperref}

\usepackage[preprint]{icml2026}

\usepackage{amsmath}
\usepackage{amsthm}
\usepackage{graphicx}
\usepackage[utf8]{inputenc}
\usepackage{amssymb}
\usepackage{amsfonts}
\usepackage{stmaryrd}
\usepackage{ragged2e}
\usepackage{amsmath}
\usepackage{epsfig}
\usepackage{color}
\usepackage{dsfont}
\usepackage{tikz}
\usepackage{bbold}
\usepackage{ytableau}
\usepackage{algorithm}
\usepackage{algorithmic}
\usepackage{mathtools}
\usetikzlibrary{positioning}

\DeclareMathOperator*{\argmin}{argmin}

\newcommand{\n}{\llbracket n \rrbracket}

\newcommand{\Sn}{\mathfrak{S}_n}

\newcommand{\crd}{\textsc{CRD}}

\usepackage[capitalize,noabbrev]{cleveref}
\usepackage[english]{babel}
\theoremstyle{plain}

\newtheorem{proposition}{Proposition}
\newtheorem{lemma}{Lemma}

\theoremstyle{definition}
\newtheorem{definition}{Definition}

\theoremstyle{remark}
\newtheorem{remark}{Remark}

\begin{document}

\twocolumn[\icmltitle{Beyond Kemeny Medians: Consensus Ranking Distributions\\ Definition, Properties and Statistical Learning}
\icmlsetsymbol{equal}{*}
\begin{icmlauthorlist}
\icmlauthor{Stephan Cl\'emen\c{c}on}{yyy}
\icmlauthor{Ekhine Irurozki}{yyy}
\end{icmlauthorlist}
\icmlaffiliation{yyy}{LTCI, T\'el\'ecom Paris, Institut Polytechnique de Paris}
\icmlcorrespondingauthor{Stephan Cl\'emen\c{c}on}{stephan.clemencon@telecom-paris.fr}
\icmlcorrespondingauthor{Ekhine Irurozki}{irurozki@telecom-paris.fr}
\icmlkeywords{Machine Learning, ICML}
\vskip 0.3in
]


\printAffiliationsAndNotice{}  



\begin{abstract}
  In this article we develop a new method for summarizing a ranking distribution, \textit{i.e.} a probability distribution on the symmetric group $\mathfrak{S}_n$, beyond the classical theory of consensus and Kemeny medians. Based on the notion of \textit{local ranking median}, we introduce the concept of \textit{consensus ranking distribution} ($\crd$), a sparse mixture model of Dirac masses on $\mathfrak{S}_n$, in order to approximate a ranking distribution with small distortion from a mass transportation perspective.  We prove that by choosing the popular Kendall $\tau$ distance as the cost function, the optimal distortion can be expressed as a function of pairwise probabilities, paving the way for the development of efficient learning methods that do not suffer from the lack of vector space structure on $\mathfrak{S}_n$.  In particular, we propose a top-down tree-structured statistical algorithm that allows for the progressive refinement of a CRD based on ranking data, from the Dirac mass at a Kemeny median at the root of the tree to the empirical ranking data distribution itself at the end of the tree's exhaustive growth. In addition to the theoretical arguments developed, the relevance of the algorithm is empirically supported by various numerical experiments.      
\end{abstract}

\section{Introduction}

 Designing machine learning algorithms tailored to preference data is key to optimizing the performance of recommendation systems and search engines. These tools, which are indispensable in the digital age, continuously use and generate ever-increasing volumes of data.
The nature of the data feeding into or produced by these algorithms, which mainly consists of (partial) rankings expressing observed or predicted preferences, continues to pose a methodological challenge. Given that the number $n!$ of possible rankings explodes with the number $n\geq 1$ of instances to be ordered,  it is crucial to develop dedicated dimensionality reduction methods in order to represent/analyze ranking data in an efficient manner. Regardless of the type of task being considered (supervised, unsupervised), the vast majority of machine learning algorithms based on multivariate data rely on the calculation of statistical quantities such as (pairwise) averages or linear combinations of features, effectively representing (location and dispersion) properties of the data distribution. Summarizing ranking distributions, \textit{i.e.} probability laws on the group of permutations $\mathfrak{S}_n$, is anything but simple. Given the absence of a vector space structure on $\mathfrak{S}_n$, extending basic concepts such as mean/median or variance, ubiquitous in multivariate data analysis, to preference data poses mathematical and computational difficulties. Permutations can be embedded in the Birkhoff polytope (\textit{i.e.} the convex hull of the set of permutation matrices), see \cite{clemenccon2010kantorovich} and \cite{LindermanMena2017reparameterizing}, or in a feature space by means of the kernel trick, see \cite{KGD18}. However, the representation of the original ranking distribution thus provided cannot be easily interpreted, the coordinates of such embeddings being highly correlated in general. Although it is obviously possible, based on the prior choice of a metric on $\mathfrak{S}_n$ (most often the Kendall $\tau$ distance, \textit{i.e.} the number of discording pairs), to define a barycentric permutation $\sigma^*$ in $\mathfrak{S}_n$, \textit{i.e.} a ranking median, and a related dispersion measure (the average distance to the median namely) from a collection of rankings, its calculation can prove difficult. This problem, known as \textit{consensus ranking} or \textit{ranking aggregation} and referring to the pioneering work of Condorcet in social choice theory, continues to be the subject of intense interest. References are too numerous to be listed exhaustively, one may refer to \textit{e.g.}  \cite{dwork2001rank}, \cite{procaccia2016optimal}, \cite{JKSO16} or \cite{JKS16} among others. In practice,  ranking medians provide a summary of a statistical population of rankings that is far too crude in many cases, and rigid statistical modeling using classical probability laws on $\mathfrak{S}_n$ (\textit{e.g.} Mallows, Placket-Luce) reaches its limits as soon as ranking data exhibit significant heterogeneity.

In this article, we develop a new framework for approximating/estimating a ranking distribution $P$, \textit{i.e.} the law of a generic r.v. $\Sigma=(\Sigma(1),\; \ldots,\; \Sigma(n))$ defined on a probability space $(\Omega,\; \mathcal{F},\; \mathbb{P})$ and valued in $\mathfrak{S}_n$, extending the concept of Kemeny medians, \textit{i.e.} ranking medians when the metric chosen on $\mathfrak{S}_n$ is the Kendall $\tau$ distance. Our approach is inspired by \textit{local averaging methods}, which are widely used in multivariate data analysis and underlie many popular methods and algorithms such as $K$-means or histogram estimation.  It is based on the concept of \textit{local median} which, combined with a partition $\mathcal{P}$ of the symmetric group and associated class probabilities $\alpha_{\mathcal{C}}=\mathbb{P}\{\Sigma \in \mathcal{C}\}$ with $\mathcal{C}\in \mathcal{P}$, allows us to define what we call a \textit{consensus ranking distribution} ($\crd$). Namely, a (parsimonious) model $P^*_{\mathcal{P}}=\sum_{\mathcal{C}\in \mathcal{P}}\alpha_{\mathcal{C}}\delta_{\sigma^*_{\mathcal{C}}}$ of a mixture of Dirac measures centred on local medians $\sigma^*_{\mathcal{C}}$ (\textit{i.e.} Kemeny medians of $\Sigma$'s conditional distribution given $\Sigma\in \mathcal{C}$ with $\mathcal{C}\in \mathcal{P}$) that approximates the ranking distribution $P$ of interest better than a simple point mass $\delta_{\sigma^*}$ at a global median $\sigma^*$ and with much less uncertainty than the raw empirical version $\widehat{P}_N=(1/N)\sum_{k\leq N}\delta_{\Sigma_k}$ of $P$ based on $N<<n!$ independent copies $\Sigma_1,\; \ldots,\; \Sigma_N$ of the r.v. $\Sigma$ would. More specifically, we demonstrate that the CRD $P^*_{\mathcal{P}}$ has attractive properties  w.r.t. approximation of the distribution $P$, provided that $\mathcal{P}$ is chosen appropriately, when distortion is measured by the optimal transport metric whose cost function is the Kendall $\tau$ distance and prove a bound for the distortion that can be expressed using the local pairwise probabilities $\mathbb{P}\{ \Sigma(i)<\Sigma(j)\mid \Sigma\in \mathcal{C}\}$ with $\mathcal{C}\in \mathcal{P}$ and $1\leq i <j\leq n$ only (just like the local Kemeny medians in certain situations). Since these quantities can be accurately estimated when $\mathcal{P}$'s cells $\mathcal{C}$ are large enough, \textit{i.e.} when the $\alpha_{\mathcal{C}}$'s are large enough, this allows us to design a a statistical learning algorithm to recover, from ranking data, a $\crd$ estimate balancing  distortion bias against uncertainty/variability.
Precisely, a recursive partioning technique based on binary decision trees operating on $\mathfrak{S}_n$ (the root node) with splitting rules of the form '$\sigma(i)<\sigma(j)$' defined by pairs $(i,j)$ of instances, we call \textsc{COAST} (COnsensus rAnking deciSion Trees). Beyond theoretical arguments, numerical experiments are presented, providing empirical evidence of the relevance of the \textsc{COAST} algorithm. 

The article is structured as follows. In section \ref{sec:background}, basic concepts and results pertaining to (statistical) consensus ranking theory are briefly recalled. The rigorous definition of consensus ranking distribution is given in section \ref{sec:crd}, together with the main (statistical) properties, while the \textsc{COAST} algorithm is introduced and (empirically) analyzed in section \ref{sec:algorithm}. Some concluding remarks are eventually collected in section \ref{sec:conclusion}. Technical details are deferred to the Appendix.
\section{Background and Preliminaries}\label{sec:background}

We introduce the main notations and recall basic notions in consensus ranking theory for clarity's sake. The framework developed in the next section relies on these concepts and can be seen as a natural extension. Throughout the paper, the indicator function of any event $\mathcal{E}$ is denoted by $\mathbb{I}\{\mathcal{E}  \}$, the Dirac mass at any point $a$ by $\delta_a$, the cardinality of any finite subset $A$ by $\#A$, the support of a ranking distribution $P$ by $\text{Supp}(P)$.
Preferences expressed among a set of $n\geq 1$ items, indexed by $i\in \n:=\{1,\; \ldots,\; n\}$ say, can be viewed (in the absence of ties for simplicity) as a permutation (full ranking) $\sigma\in \Sn$ that maps any item $i$ to its rank $\sigma(i)$. Certain marginals of the law $P$ of a random permutation $\Sigma$ on $\Sn$ are of specific interest. In particular, for a pair of items $(i,j) \in \n^2$, the quantity $p_{i,j}=\mathbb{P}\{\Sigma(i)<\Sigma(j)\}$ for $\Sigma \sim P$ is referred to as a pairwise marginal of $P$ and indicates the probability that item $i$ is preferred to (ranked lower than) item $j$. Note that $p_{i,j}+p_{j,i}=1$.

\subsection{Consensus Ranking and Kemeny Medians}

\textit{Consensus ranking}, also referred to as \textit{ranking aggregation},is a problem than can be easily stated in an informal manner. Let $\sigma_{1},\; \ldots,\; \sigma_{N}$ be a collection of $N\geq 1$ (full) rankings on $\n$. The goal is to find a ranking $\sigma^*\in \Sn$ that summarizes it best. A rigorous and quantitative version of it can be formulated using a metric $d(.\,,\,.)$ on $\Sn$ as the minimization problem:
\begin{equation}	\label{eq:ranking_aggregation}
	\min_{\sigma\in \mathfrak{S}_n}\sum_{k=1}^N d(\sigma,\sigma_{k}),
\end{equation}
Solutions to \eqref{eq:ranking_aggregation} always exist of course (since $\#\Sn<+\infty$) but are not necessarily unique. Such barycentric permutations are called \textit{consensus/median ranking}. Various distances have been considered in the consensus ranking literature, \textit{e.g.} Spearman $\rho$, Spearman footrule, Hamming, see Chapter 11 in \cite{Deza}. The most widely documented version of the problem is Kemeny ranking aggregation \cite{Kemeny59}, that corresponds to the case where the metric chosen is the Kendall $\tau$ distance, \textit{i.e.} the number of pairwise disagreements:
\begin{equation*}\label{eq:Kendall_tau}
d_{\tau}(\sigma,\sigma')=\sum_{i<j}\mathbb{I}\{(\sigma(i)-\sigma(j)) (\sigma'(i)-\sigma'(j))<0\}
\end{equation*}
for all $(\sigma,\sigma')\in \mathfrak{S}_n^2$. From a statistical learning perspective, \eqref{eq:ranking_aggregation} can be viewed as a $M$-estimation problem if one assumes that the rankings to be aggregated/summarized is composed of $N\geq 1$ independent copies $\Sigma_1,\; \ldots,\; \Sigma_N$ of a generic r.v. $\Sigma$, defined on a probability space $(\Omega,\; \mathcal{F},\; \mathbb{P})$ and distributed according to an unknown probability distribution $P$ on $\mathfrak{S}_n$ (\textit{i.e.} $P(\sigma)=\mathbb{P}\{ \Sigma=\sigma \}$ for all $\sigma\in \mathfrak{S}_n$). Similar to a median of a real valued r.v. $Z$, which is any scalar closest to $Z$ in the $L_1$ sense, a (true) ranking median of distribution $P$ w.r.t. a certain metric $d$ on $\mathfrak{S}_n$ is any solution of
\begin{equation}\label{eq:median_pb}
\min_{\sigma \in \mathfrak{S}_n}L_P(\sigma),
\end{equation}
where the \textit{ranking risk} $L_P(\sigma)=\mathbb{E}_{\Sigma \sim P}[d(\Sigma,\sigma)  ]
$ is the expected distance between any permutation $\sigma$ and the r.v. $\Sigma$. Statistical ranking aggregation consists in recovering a solution $\sigma^*$ of this minimization problem, plus an estimate of this minimum $L^*_P=L_P(\sigma^*)$, as accurately as possible, from the training examples $\Sigma_1,\; \ldots,\; \Sigma_N$. A ranking median $\sigma^*$ can be viewed as a location parameter, a central value, for the ranking distribution $P$, and the quantity $L^*_P$ as a dispersion measure.
Like \eqref{eq:ranking_aggregation}, the minimization problem \eqref{eq:median_pb} has always a solution but can be multimodal.
A major obstacle is of course that the ranking risk $L_P(.)$ is unknown, as the ranking distribution $P$. In absence of specific model assumptions about $P$ (such as \textit{e.g.} the Mallows model), when the ranking aggregation task is based on the $\Sigma_k$'s only, the Empirical Risk Minimization (ERM) paradigm, see \cite{Vapnik}, encourages us to solve the statistical version of the problem obtained by replacing $P$ in \eqref{eq:median_pb} with
$\widehat{P}_N$: $\min_{\sigma\in \mathfrak{S}_n}\widehat{L}_N(\sigma)$ with
\begin{equation}\label{eq:emp_risk}
\widehat{L}_N(\sigma)=(1/N)\sum_{k=1}^N d(\Sigma_k,\sigma)=L_{\widehat{P}_N}(\sigma).
\end{equation}

The performance of empirical minimizers $\widehat{\sigma}_N$
 is studied in \cite{CKS17}: nonasymptotic bounds of order $O_{\mathbb{P}}(1/\sqrt{N})$ for the excess of risk $L_P(\widehat{\sigma}_N)-L^*_P$ in probability/expectation have been proved and shown to be sharp in the minimax sense, when $d=d_{\tau}$. If ranking aggregation \eqref{eq:ranking_aggregation} is NP-hard in general, see \textit{e.g.} \cite{Hudry08}, exact solutions, called \textit{Kemeny medians}, can be explicited  in the Kendall $\tau$ case, when the pairwise probabilities $p_{i,j}$, $1\leq i\neq j\leq n$, fulfill the \textit{stochastic transitivity} property.
 \begin{definition}\label{def:stoch_trans} Let $P$ be a probability law on $\mathfrak{S}_n$.
 It is said to be (weakly) stochastically transitive iff
$\forall (i,j,k)\in \n^3$,   $$p_{i,j}\geq 1/2 \text{ and } p_{j,k}\geq 1/2 \; \Rightarrow\; p_{i,k}\geq 1/2.
 $$
 If, in addition, $p_{i,j}\neq 1/2$ for all $i<j$, $P$ is said to be strictly stochastically transitive (SST).
 \end{definition}
Originally introduced in Psychology, see \cite{davidson1959experimental} and \cite{fishburn1973binary}, these conditions were exploited more recently for ranking based on pairwise comparisons, refer to \cite{shah2015stochastically} or \cite{shah2015simple}). They are satisfied by the vast majority of popular parametric models such as Mallows \cite{Mallows57} or Bradley-Terry-Luce-Plackett \cite{Plackett75} models for instance. It is shown in \cite{CKS17} (see Theorem 5 therein) that, when stochastic transitivity is fulfilled, the ensemble of Kemeny medians is
$\{\sigma\in \mathfrak{S}_n:\; (p_{i,j}-1/2)(\sigma(j)-\sigma(i ))>0 \text{ for all } i<j \text{ s.t. } p_{i,j}\neq 1/2  \}$, and
 the related dispersion measure, denoted by $V_P:=L^*_P$ for clarity, is given by
 \begin{equation*}\label{eq:inf}
 V_P=\sum_{i<j}\min\{p_{i,j},1-p_{i,j}  \}=\sum_{i<j}\{1/2-\vert  p_{i,j}-1/2\vert\}.
 \end{equation*}
  \begin{remark}{\sc (Ranking variability)} Except in the situation recalled above, calculating $L^*_P$ is difficult. Since an expectation is generally easier to calculate/approximate than a minimum, an alternative dispersion measure is given by:
  \begin{equation}
V'_P=\mathbb{E}[d(\Sigma,\Sigma')]/2,
  \end{equation}
  where $\Sigma'$ means an independent copy of $\Sigma$. Note that $V'_P\leq V_P \leq 2V'_P$ and, when $d=d_{\tau}$, $V'_P=\sum_{i<j}p_{i,j}(1-p_{i,j})$.
  \end{remark}
 When $P$ is SST, the Kemeny median $\sigma^*_P$ is unique and matches the Copeland ranking:
 \begin{equation}\label{eq:sol_SST}
\forall i\in \n,\;\;  \sigma^*_P(i)=1+\sum_{j\neq i}\mathbb{I}\{p_{i,j}<1/2  \}.
 \end{equation}
If $P$ additionally fulfills the low-noise condition {\bf NA}$(h)$, defined for $h\in (0,1/2)$ by
\begin{equation}\label{eq:hyp_margin0}
 \min_{i<j}\left\vert p_{i,j}-1/2 \right\vert \ge h,
\end{equation}
statistical Kemeny ranking aggregation is a learning task that can be completed with a fast rate of order $O_{\mathbb{P}}(1/N)$, see Proposition 14 in \cite{CKS17}. Condition 
{\bf NA}$(h)$ can be considered analogous to low-noise conditions in binary classification, see \cite{KB05}, and provides a control of ranking variability, insofar as it implies: $V_P\leq \binom{n}{2}(1/2-h)$.

\subsection{Beyond Ranking Medians}\label{subsec:beyond}

Although very popular, once resolved, the consensus problem offers only a very crude summary of a ranking distribution. Other quantities have recently been defined in an attempt to describe certain properties of a ranking distribution more exhaustively, and issues relating to their statistical estimation have been explored. Because they easily relate to the concept we develop in the following section as will be seen later, we mention two such notions here.

\noindent {\bf Ranking depth.} Ranking medians should be naturally considered as measures of central tendency in the ranking context. Following on from this idea, the concept of \textit{ranking depth} has been introduced in \cite{pmlr-v151-goibert22a}:
\begin{equation*}\label{eq:ranking_depth}
\forall \sigma \in \mathfrak{S}_n,\;\; D_P(\sigma)=\max_{(\sigma_1,\sigma_2)\in \mathfrak{S}_n^2}d(\sigma_1,\sigma_2)-L_P(\sigma).
\end{equation*}
Equipped with this notion for quantifying the centrality of any $\sigma\in \mathfrak{S}_n$, ranking medians are the deepest permutations and a center-outward ordering on the symmetric group is defined, which allows the concepts of quantiles, ranks and the powerful statistical methods in multivariate data analysis based on them to be directly extended to ranking data analysis. Refer also to \cite{goibert2023robust} for applications to robust ranking.

\noindent {\bf Bucket ranking distribution.} As noted in \cite{achab2019bucket}, $L_P(\sigma)$ can be seen as a distance between the ranking distributions $P$ and $\delta_{\sigma}$ when considering the Wasserstein metric defined below.

\begin{definition}
Let $d:\Sn^2\rightarrow \mathbb{R}_+$ be a distance on $\Sn$. The Wasserstein metric with $d$ as cost function between two probability distributions $P$ and $P'$ on $\mathfrak{S}_n$ is given by:
\begin{equation} \label{eq:metric}
W_{d}\left(P,P'  \right)=\inf_{\Sigma\sim P,\; \Sigma' \sim P' }\mathbb{E}\left[ d(\Sigma,\Sigma') \right],
\end{equation}
where the infimum is taken over all possible couplings\footnote{By coupling of two probability distributions $Q$ and $Q'$ is meant a random pair $(U,U')$ s.t. the marginal distributions of $U$ and $U'$ are $Q$ and $Q'$.} $(\Sigma,\Sigma')$ of $(P,P')$.
\end{definition}
Consensus ranking can be thus viewed as a radical \textit{dimensionality reduction} procedure, summarizing $P$ by its closest Dirac measure $\delta_{\sigma^*_P}$ w.r.t. to the probability metric above leading to the distortion $L^*_P=W_d(P,\delta_{\sigma^*_P})$. In \cite{achab2019bucket}, a more general framework for dimensionality reduction has been developed, which can be seen as an extension of consensus ranking. Under the initial hypothesis that the set $\n$ of instances to be ranked can be partitioned into \textit{buckets} such that items in a certain bucket are either all ranked higher or else all ranked lower than items in another bucket with high probability, it permits to define a specific low dimensional approximation $\tilde{P}$ of a ranking distribution $P$, referred to as a \textit{bucket ranking distribution}. Special attention is paid to the situation where the distortion is measured by the Kendall Wasserstein metric $W_{d_{\tau}}(\tilde{P},P)$, since the minimum distortion can be then expressed as a function of the pairwise probabilities $p_{i,j}$ only (see Proposition 5 in \cite{achab2019bucket}) and bucket ranking distributions with small distortion can be directly built by 'randomizing' empirical Kemeny medians, see section 3 in \cite{achab2019bucket} for further details.

The concepts mentioned above are relevant in the unimodal case especially. The definitions and algorithms presented and analyzed in the next sections are part of this line of research, with the notable exception that they relate to multimodal and/or spread-off ranking distributions $P$ and crucially require the introduction of a notion of localization.

\section{Consensus Ranking Distributions} \label{sec:crd}

In this section we develop a framework to approximate a possibly multimodal ranking distribution $P$, whose variability $L^*_P$ is too large for an approximation by a simple Dirac mass $\delta_{\sigma^*_P}$ at a median to be satisfactory. After introducing the concept of consensus (ranking) distribution (CRD in abbreviated form) and describing the inherent trade-off between dimension reduction and distortion minimization, we then study the performance of statistical versions of CRD's from the perspective of empirical risk minimization. 

\subsection{Local Ranking Medians and Partitions}\label{subsec:crd}
The statistical recovery of a ranking distribution $P$ based on an i.i.d. sample of $N\geq 1$ observations $\Sigma_1,\; \ldots,\; \Sigma_N$ is challenging because its dimensionality is $n!-1$ in absence of restrictive parametric model assumptions:
$P=\sum_{\sigma\in\mathfrak{S}_n}P(\sigma)\delta_{\sigma}$ with $\sum_{\sigma\in\mathfrak{S}_n}P(\sigma)=+1$. Since $n!$ is generally very large compared to the ranking dataset size $N$ in practice, the raw empirical distribution $\widehat{P}_N=\sum_{\sigma\in\mathfrak{S}_n}\widehat{P}_N(\sigma)\delta_{\sigma}$ does not provide an accurate estimation of $P$. To avoid the rigidity of the approach consisting of specifying a  statistical model with a low dimension relative to $n!-1$ (typically of order $n$ such as the Mallows or Plackett-Luce models, see \cite{Marden96}), 
 we develop an alternative inspired by \textit{local averaging} techniques used for nonparametric statistics of multivariate data. The approximation/estimation methodology we propose is based on the localization of the ranking median concept defined below.

\begin{definition} {\sc (Local consensus)} Let $\Sigma$ be a random permutation of $\n$ with law $P$ and $\mathcal{C}\subset \mathfrak{S}_n$ s.t. $P(\mathcal{C})>0$. Denote $\Sigma$'s conditional distribution given $\Sigma\in \mathcal{C}$ by $P_{\mathcal{C}}$. By local median $\sigma^*_{\mathcal{C}}\in \mathfrak{S}_n$ of the distribution $P$ on region $\mathcal{C}$ is meant any ranking median of $P_{\mathcal{C}}$: $\mathbb{E}[d(\Sigma, \sigma^*_{\mathcal{C}})\mid \Sigma\in \mathcal{C}]=V(\mathcal{C})$, denoting $P$'s local variability on $\mathcal{C}$ by
$$
V(\mathcal{C}):=\min_{\sigma\in \mathfrak{S}_n}L_{P_{\mathcal{C}}}(\sigma)=V_{P_{\mathcal{C}}}.
$$
\end{definition}
Note that a local median $\sigma^*_{\mathcal{C}}$ of $P$ on a region $\mathcal{C}\subset \mathfrak{S}_n$ does not necessarily belong to $\mathcal{C}$ (see the example in the Supplementary Material). We therefore propose defining a probabilistic notion of consensus, we call consensus ranking distribution, as a convex combination of Dirac masses at local medians on regions forming a partition of $\mathfrak{S}_n$.
\begin{definition} {\sc (Consensus distribution)} Let $\mathcal{P}$ be a partition\footnote{Recall that by partition $\mathcal{P}$ of a set $E$ is meant any collection of non empty pairwise disjoint subsets $\mathcal{C}$ of $E$ s.t. $\cup_{\mathcal{C}\in \mathcal{P}}\mathcal{C}=E$.} of $\mathfrak{S}_n$ and $P$ be a ranking distribution. A consensus distribution of $P$ relative to $\mathcal{P}$ is any ranking distribution of the form
\begin{equation}\label{eq:def_crd}
P_{\mathcal{P}}=\sum_{\mathcal{C}\in \mathcal{P}}P(\mathcal{C})\delta_{\sigma^*_{\mathcal{C}}},
\end{equation}
where $\sigma^*_{\mathcal{C}}$ is a local median of $P$ on $\mathcal{C}\in \mathcal{P}$.
\end{definition}
The approximant \eqref{eq:def_crd} is a generalization of the two 'extreme' cases constituted by the Dirac mass $P(\mathfrak{S}_n)\delta_{\sigma^*_P}$ at a median, which is of minimal dimension but has maximum distortion, and by the discrete law $P=\sum_{\sigma\in \mathfrak{S}_n}P(\sigma)\delta_{\sigma}$ itself which has zero distortion but is of dimension $n!-1$. The choice of a consensus distribution is therefore a trade-off: the partition $\mathcal{P}$ must be fine enough so that the distortion of $P_{\mathcal{P}}$ is sufficiently small, while at the same time its cells $\mathcal{C}$ must be large enough for the statistics $\widehat{P}_N(\mathcal{C})$ based on $N\geq 1$ observations to be accurate estimates of $P(\mathcal{C})$, resulting in a reduced dimensionality of \eqref{eq:def_crd}. The result below illustrates the balance, by providing a control of the distortion of \eqref{eq:def_crd} by the intra-cell ranking variability.

\begin{proposition}\label{prop:bound}{\sc (Distortion bound)} Let $\mathcal{P}$ be any partition of $\mathfrak{S}_n$ s.t. $P(\mathcal{C})>0$ for all $\mathcal{C}\in \mathcal{P}$.
We have:
\begin{eqnarray}
W_{d}(P,P_{\mathcal{P}})&\leq& \mathcal{E}_P(\mathcal{P}):=\sum_{\mathcal{C}\in \mathcal{P}}P(\mathcal{C})V(\mathcal{C})\label{eq:bound1}\\ \notag
&\leq& 2\mathcal{E}'_P(\mathcal{P}):=2\sum_{\mathcal{C}\in \mathcal{P}}P(\mathcal{C})V'(\mathcal{C}),
\end{eqnarray}
where we set $V'(\mathcal{C})=V'_{P_{\mathcal{C}}}$.
\end{proposition}

\begin{proof}
Consider the coupling $(\Sigma,\Sigma_{\mathcal{P}})$ of $(P,P_{\mathcal{P}})$ defined by $\mathbb{P}\{\Sigma_{\mathcal{P}}=\sigma^*_{\mathcal{C}}\mid \Sigma\in \mathcal{C}\}=1$ for all $\mathcal{C}\in \mathcal{P}$.
The distortion can be then bounded as follows
\begin{multline*}\label{eq:bound1}
W_{\tau}(P_{\mathcal{P}},P)\leq \mathbb{E}[d(\Sigma, \Sigma_{\mathcal{P}})]= \sum_{\mathcal{C}\in \mathcal{P}}P(\mathcal{C})V(\mathcal{C}),
\end{multline*}
using the total probability rule and the second bound is simply due to the fact that $V(\mathcal{C})\leq V'(\mathcal{C})$ for all $\mathcal{C}\in \mathcal{P}$.
\end{proof}
The coupling used in the argument above is not always optimal, meaning that \eqref{eq:bound1} may be strict, as shown in the Supplementary Material (see the counterexample therein). However, there is equality in the case of the least fine partition $\mathcal{P}_1:=\{\mathfrak{S}_n\}$ and in the case of the finest partition $\mathcal{P}_{n!}:=\cup_{\sigma\in \text{Supp}(P)}\{\sigma\}$. 
The same argument can be used to show that the difference between $P$'s variability and that of its CRD $P_{\mathcal{P}}$ is controlled by the distortion, namely $\vert V_{\mathcal{P}}-V_{P_{\mathcal{P}}}\vert \leq W_d(P,P_{\mathcal{P}})$, and that the depth of any ranking median $\sigma^*_{P_{\mathcal{P}}}$ w.r.t. $P$ is close to the maximum depth when the distortion is small: $ D_P(\sigma^*_{P_{\mathcal{P}}})\geq \max_{\sigma\in \mathfrak{S}_n}D_P(\sigma)-W_d(P,P_{\mathcal{P}})$.
Observe also that
\begin{equation}\label{eq:bound2}
\mathcal{E}_P(\mathcal{P})\leq \min_{\sigma\in \mathfrak{S}_n}\left\{ \sum_{\mathcal{C}\in \mathcal{P}}P(\mathcal{C})\mathbb{E}_{P_{\mathcal{C}}}[d(\sigma,\Sigma)] \right\}= V_{P},
\end{equation}
and, similarly, $\mathcal{E}'_P(\mathcal{P})\leq V'_{P}$ for any partition $\mathcal{P}$. More generally, the finer the partition, the smaller the intra-cell variability, we have: $\mathcal{E}_P(\mathcal{P}')\leq \mathcal{E}_P(\mathcal{P})$ (respectively, $\mathcal{E}'_P(\mathcal{P}')\leq \mathcal{E}'_P(\mathcal{P})$ ) as soon as $\mathcal{P}'$ is a subpartition of $\mathcal{P}$, \textit{i.e.} $\mathcal{P}'$ is a partition of $\mathfrak{S}_n$ such that, for all $\mathcal{C}'\in \mathcal{P}'$, there exists $\mathcal{C}\in \mathcal{P}$ s.t. $\mathcal{C}'\subset \mathcal{C}$. Of course, to adapt to the properties of the ranking distribution $P$, the partition $\mathcal{P}$ defining the CRD $P_{\mathcal{P}}$ must be learned from the data $\Sigma_1,\; \ldots,\; \Sigma_N$.
Because it is easy to compute a statistical counterpart of $\mathcal{E}'(\mathcal{P})$ based on the $\Sigma_k$'s, the bound in Proposition \ref{prop:bound} permit to choose empirically a partition of controlled cardinality yielding a CRD with guarantees regarding distortion in the form of confidence bounds. It is the purpose of the subsection below to establish such generalization bounds.

\begin{remark}{\sc (Nonparametric inspiration)} While the framework we develop here is based on partioning techniques just like histogram methods in multivariate data analysis, other approaches to statistical analysis of ranking data have been directly inspired by nonparametric techniques tailored to observations in $\mathbb{R}^q$ with $q\geq 1$. One may refer to \textit{e.g.} \cite{Mao2008} for the development of a method analogous to kernel smoothing for ranking distribution inference and to \cite{Hudry08}, \cite{kondor2010ranking} or \cite{pmlr-v37-sibony15} for “projection-like” approaches.
\end{remark}

\subsection{Learning $\crd$'s - Statistical Guarantees}

Based on Proposition \ref{prop:bound}, we now study the principle of choosing the partition characterizing the CRD \eqref{eq:def_crd} by minimizing a statistical version of the quantity $\mathcal{E}'(\mathcal{P})$ over a finite collection $\Pi_K$ of partitions $\mathcal{P}$ of $\mathfrak{S}_n$ with $K\leq n!$ cells $\mathcal{C}$ s.t. $P(\mathcal{C})>0$ from the perspective of generalization. Observe first that, based on the observations $\Sigma_1,\; \ldots,\; \Sigma_N$ supposedly available, a natural empirical counterpart of $\mathcal{E}'(\mathcal{P})$ is obtained in a plug-in fashion by replacing $P$ with $\widehat{P}_N$ and is given by:
\begin{equation}\label{eq:emp_crit}
\widehat{\mathcal{E}}_N'(\mathcal{P})=\sum_{\mathcal{C}\in \mathcal{P}}\widehat{P}_N(\mathcal{C})\widehat{V}'_N(\mathcal{C}),
\end{equation}
where the local ranking variability estimates are 
\begin{equation*}
\widehat{V}'_N(\mathcal{C})=\frac{2}{N_{\mathcal{C}}(N_{\mathcal{C}}-1)}\sum_{i<j}\mathbb{I}\{(\Sigma_i,\Sigma_j)\in \mathcal{C}^2\}d(\Sigma_i,\Sigma_j),
\end{equation*}
with $N_{\mathcal{C}}=\#\{i\in\{1,\; \ldots,\; N\}:\; \Sigma_i \in \mathcal{C} \}$ and the convention that $\widehat{V}'_N(\mathcal{C})=0$ when $N_{\mathcal{C}}\leq 1$, where $\mathcal{C}\in \mathcal{P}$. The fluctuations of the stochastic process $\{\widehat{\mathcal{E}}_N'(\mathcal{P})-\mathcal{E}(\mathcal{P}):\; \mathcal{P}\in \Pi_K\}$ can be controlled by observing that: $\forall \mathcal{P}\in \Pi_K$,
\begin{multline}\label{eq:decomp}
\widehat{\mathcal{E}}_N'(\mathcal{P})-\mathcal{E}(\mathcal{P})=\left(\widehat{\mathcal{E}}_N'(\mathcal{P})-\widetilde{\mathcal{E}}_N'(\mathcal{P})\right)\\
+\left(\widetilde{\mathcal{E}}_N'(\mathcal{P})-\mathcal{E}(\mathcal{P})\right),
\end{multline}
where we set
\begin{equation}\label{eq:U_stat}
\widetilde{\mathcal{E}}_N'(\mathcal{P})=\frac{2}{N(N-1)}\sum_{i<j}\Phi_{\mathcal{P}}(\Sigma_i,\Sigma_j),
\end{equation}
with $$
\Phi_{\mathcal{P}}(\sigma,\sigma')=\sum_{\mathcal{C}\in \mathcal{P}}\frac{\mathbb{I}\{(\sigma,\sigma')\in \mathcal{C}^2\}}{P(\mathcal{C})}\cdot d(\sigma,\sigma').$$
Hence, the quantity \eqref{eq:U_stat} is a (non degenerate) $U$-statistic of degree $2$ with symmetric kernel $\Phi_{\mathcal{P}}$, see \cite{Lee90}. The maximal deviations of the second term on the right hand side of \eqref{eq:decomp} can be thus controlled by means of classical extensions of tail inequalities for i.i.d. averages to $U$-statistics. Since those of the first term can be easily shown to be of order $O_{\mathbb{P}}(1/\sqrt{N})$, we have the following result (see technical details in the Appendix).

\begin{proposition}\label{prop:gen_bound}
Let $K\in \{1,\; \ldots,\; n!\}$ and $\Pi_K$ be a collection of partitions of $\mathfrak{S}_n$ counting each $K$ cells $\mathcal{C}$ s.t. $P(\mathcal{C})\geq \chi(\Pi_K)>0$. Consider any minimizer $\widehat{\mathcal{P}}_N$ of the empirical criterion \eqref{eq:emp_crit} over $\Pi_K$. For all $\delta\in (0,1)$, we have with probability at least $1-\delta$:
\begin{multline*}
W_d(P, P_{\widehat{\mathcal{P}}_N})/2\leq \mathcal{E}'(\widehat{\mathcal{P}}_N)\leq \min_{\mathcal{P}\in \Pi_K}\mathcal{E}'(\mathcal{P})+\\
 2\sqrt{\frac{\gamma(\Pi_K)}{\chi(\Pi_K)}\frac{\log(4\# \Pi_K/\delta)}{N}}+ 2C\frac{K\vert\vert d\vert\vert_{\infty}}{\chi^2(\Pi_K)}\sqrt{\frac{\log(8K/\delta)}{2N}},
\end{multline*}
if $N\geq \max\{\log(8K/\delta),(1+\sqrt{\log(4K/\delta)/2})/\chi(\Pi_K)\}$,
where we set
$$\gamma(\Pi_K)=\max \left\{d(\sigma,\sigma'): (\sigma,\sigma')\in \mathcal{C}^2, \mathcal{C}\in \mathcal{P}, \mathcal{P}\in \Pi_K   \right\}.$$
\end{proposition}

The bound above provides guarantees for the principle consisting in choosing the partition defining the CRD by minimizing \eqref{eq:emp_crit} over all $\mathcal{P}$ in the collection $\Pi_K$ considered. As shown in the Appendix, a similar guarantee holds true, under appropriate conditions, for statistical versions of $P_{\widehat{\mathcal{P}}_N}$, where local medians $\sigma^*_{\mathcal{C}}$ and weights $P(\mathcal{C})$ are replaced with their empirical counterparts, see Proposition \ref{prop:gen_bound2} therein.

{\bf Choosing the number $K$ of local medians.} A crucial model selection issue remains to be solved: determining an appropriate value for the number $K\in \{1,\; \ldots,\; n!\}$ of local medians or cells involved in \eqref{eq:def_crd}, a natural proxy for the 'complexity' of  the resulting $\crd$, so as to achieve a satisfactory 'distortion bias \textit{vs} variance' trade-off. The popular idea of complexity regularization method consists in adding a penalty term $\textsc{pen}(N,m)$ to $\mathcal{E}(\widehat{\mathcal{P}})$. Consider a sequence $\Pi_1,\; \ldots,\; \Pi_{M}$ of $M\geq 1$ collections  of partitions of $\mathfrak{S}_n$, $\Pi_m$ containing only partitions of cardinality $m$ and let $\widehat{\mathcal{P}}_N^{(m)}$ be a minimizer of the empirical criterion \eqref{eq:emp_crit} over $\Pi_m$. One then selects $\widehat{\mathcal{P}}_N^{(\widehat{m})}$ where the index $\widehat{m}$ minimizes $\widehat{\mathcal{E}}'(\widehat{\mathcal{P}}_N^{(m)})+\textsc{pen}(N,m)$ over $\{1,\; \ldots,\; M\}$ to avoid overfitting. A classic approach consists in choosing for $\textsc{pen(N,m)}$ an upper confidence bound for $\sup_{\mathcal{P}\in \Pi_M}\vert \widehat{\mathcal{E}}(\mathcal{P})-\mathcal{E}(\mathcal{P}) \vert$. A (deterministic) bound can be derived from the argument of Proposition \ref{prop:gen_bound} for instance. Following line by line the reasoning in section 8 of \cite{boucheron2005theory}, it can be shown that the partition thus selected defines a $\crd$ that nearly achieves the distortion that would be obtained with the help of an oracle, revealing the value $m^*$ minimizing $\mathbb{E}[\mathcal{E}'(\widehat{\mathcal{P}}_N^{(m)})]$. Due to space limitations, details are left to the reader.


Of course, the number of cells is not the only crucial hyperparameter for the method analyzed above from a statistical perspective to produce an appropriate estimated CRD. The geometry of the cells $\mathcal{C}$ plays a major role.  It is, of course, also related to the (approximate/incremental) resolution of the optimization problem $\min_{\mathcal{P}\in \Pi}\widehat{\mathcal{E}}'_N(\mathcal{P})$, an issue we address in the next section.

\section{Partitioning Trees and CRD's}\label{sec:algorithm}

Here, we describe at length a data-driven strategy to build a nested collection of partitions, defining a sequence of CRD's of increasing complexity (\textit{i.e.} involving an increasing number of local medians) based on decision trees and incremental minimization of the $\widehat{\mathcal{E}}_N'$ criterion, referred to as the COAST (COnsensus rAnking diStribution Tree) algorithm. Next we present numerical experiments revealing its capacity to detect the local modes of ranking distributions.

\subsection{The COAST Algorithm}

Here we focus on CRDs defined by binary trees and the Kendall $\tau$ distance, for reasons that we now explain. First, top-down tree-based techniques are popular for building partitions recursively. To grow the binary tree, one starts with the root node, corresponding to the trivial partition $\mathcal{P}_1=\{\mathfrak{S}_n\}$ and one first splits $\mathfrak{S}_n$ into two disjoint regions. They are next each divided in two and the process continues until a stopping rule is applied. At each iteration, the terminal leaves of the current binary tree are in 1-to-1 correspondence with the cells of a partition of $\mathfrak{S}_n$. The main difficulty lies in choosing an appropriate \textit{binary splitting rule} (to divide a region in two sub-regions) in order to construct a partition that is well suited to the properties of the distribution $P$ of the random ranking $\Sigma$ being analyzed. We propose splitting rules of the form '$\sigma(i)<\sigma(j)$' \textit{vs} '$\sigma(i)>\sigma(j)$' determined by a pair $(i,j)$ s.t. $1\leq i<j\leq n$, referred to as a \textit{splitting pair}. The pair $(i,j)$ is used to split a cell $\mathcal{C}$ of the current partition of $\mathfrak{S}_n$ into $$\mathcal{C}^{(0)}_{i,j}:=\{\sigma\in \mathcal{C}:\; \sigma(i)<\sigma(j)\} \text{ and } \mathcal{C}^{(1)}_{i,j}:=\mathcal{C}\setminus \mathcal{C}^{(0)}_{i,j}.
$$
Each cell $\mathcal{C}$ of a partition $\mathcal{P}$ constructed this way defines a partial order on $\mathfrak{S}_n$ (and a total order when $\mathcal{C}$ reduces to a singleton of course). Hence, once a pair $(i,j)$ is chosen to split a cell $\mathcal{C}$, it will never be used again to split one of its descendants. In particular, the depth of any terminal leaf is necessarily smaller than $n(n-1)/2$, the total number of pairs. In addition to not having been used previously by these ancestors, the splitting pair $(i,j)$ related to an internal node $\mathcal{C}$ must be compatible with the partial order it defines, in the sense that the transitivity property must be satisfied. Otherwise, one of its child node is necessarily empty. Thus, although there are $n(n-1)/2$ possible splits for the root node, the number of candidate pairs decreases very rapidly as the tree is constructed.
Note that any node/cell $\mathcal{C}$ of the tree/partition takes the form
\begin{equation}\label{eq:cell_form}
\mathcal{C}=\left\{\sigma\in \mathfrak{S}_n:\; \sigma(i)<\sigma(j) \text{ for all } (i,j)\in \mathcal{I} \right\},
\end{equation}
where $\mathcal{I} \subset \{(i,j):\; 1\leq i\neq j\leq n\}$ is the set of $2$-tuples used to build $\mathcal{C}$. If the node is at depth $D$, we have $\# \mathcal{I}=D$. The reason for choosing splitting rules based on pairwise comparisons is that, if the pairwise probabilities $p_{i,j}$ with $i<j$ do not characterize the ranking distribution $P$ in general, the local versions
\begin{equation}\label{eq:local_pair}
p_{i,j}(\mathcal{C})=\mathbb{P}\{\Sigma(i)<\Sigma(j) \mid \Sigma\in \mathcal{C}\},
\end{equation}
with $\mathcal{C}$ of the form \eqref{eq:cell_form} and s.t. $P(\mathcal{C})>0$, do (see the Appendix for a rigorous formulation). In addition, when considering the notion of ranking variability based on the Kendall $\tau$ distance, the criterion $\mathcal{E}_P(\mathcal{P})=\sum_{\mathcal{C}\in \mathcal{P}}P(\mathcal{C})V(\mathcal{C})$ (respectively, $\mathcal{E}'_P(\mathcal{P})=\sum_{\mathcal{C}\in \mathcal{P}}P(\mathcal{C})V'(\mathcal{C})$) can be bounded as a function of such quantities. More generally, we have, for any $\mathcal{C}\subset \mathfrak{S}_n$ s.t. $P(\mathcal{C})>0$:
\begin{eqnarray*}
V(\mathcal{C})\leq V''(\mathcal{C})&:=&\sum_{i<j}\min\{p_{i,j}(\mathcal{C}),1-p_{i,j}(\mathcal{C})\},\\
V'(\mathcal{C})&= &\sum_{i<j}p_{i,j}(\mathcal{C})(1-p_{i,j}(\mathcal{C})).
\end{eqnarray*}
By convention, set $V(\mathcal{C})=V'(\mathcal{C})=V''(\mathcal{C})=0$ when $P(\mathcal{C})=0$. 
This allows us to define rules for choosing best the splitting pair to divide a cell $\mathcal{C}$ corresponding to an internal node of the tree: among the set $\mathcal{J}(\mathcal{C})\subset \{(i,j): 1\leq i<j\leq n \}$ of \textit{admissible} pairs, \textit{i.e.} pairs that are compatible with the partial order defining $\mathcal{C}$, choose the splitting pair $(i_{\mathcal{C}},j_{\mathcal{C}})$ to be a pair $(i,j)$ minimizing the local dispersion criterion
\begin{equation}\label{eq:split_crit}
\mathcal{E}'_{(i,j)}(\mathcal{C}):=P(\mathcal{C}_{i,j}^{(0)})V'(\mathcal{C}_{i,j}^{(0)})+P(\mathcal{C}_{i,j}^{(1)})V'(\mathcal{C}_{i,j}^{(1)}).
\end{equation}
An alternative splitting rule would consist in minimizing the version $\mathcal{E}''_{(i,j)}$ of the criterion \eqref{eq:split_crit} above, where $V'$ is replaced with $V''$. As proved in the Appendix, for any partition $\mathcal{P}$ of $\mathfrak{S}_n$, we have $\mathcal{E}''(\mathcal{P}):=\sum_{\mathcal{C}\in \mathcal{P}}P(\mathcal{C})V''(\mathcal{C})\leq V''_{P}$.
Using the same argument as that yielding \eqref{eq:bound2}: $\mathcal{E}'_{(i,j)}(\mathcal{C})\leq P(\mathcal{C})V'(\mathcal{C})$ for all $(i,j)\in \mathcal{J}_{\mathcal{C}}$. We also have $\mathcal{J}(\mathcal{C}_{i_{\mathcal{C}},j_{\mathcal{C}}}^{(m)})\subset \mathcal{J}(\mathcal{C})\setminus\{(i_{\mathcal{C}},j_{\mathcal{C}})\}$ for $m\in\{0,1\}$, as well as
$p_{i_{\mathcal{C}},j_{\mathcal{C}}}(\mathcal{C}_{i_{\mathcal{C}},j_{\mathcal{C}}}^{(0)})=1-p_{i_{\mathcal{C}},j_{\mathcal{C}}}(\mathcal{C}_{i_{\mathcal{C}},j_{\mathcal{C}}}^{(1)})=1$.
In the absence of a stopping rule, if we recursively apply this splitting procedure starting from partition $\mathcal{P}_1$, one obtains a sequence of partitions $\{\mathcal{P}_k:\; k=1,\; \ldots,\; n!\}$ such that $\#\mathcal{P}_k=k$ and $\mathcal{P}_k$ is a subpartition of $\mathcal{P}_{k-1}$ for $k\in \{2,\; \ldots,\; n!\}$. The dispersion criterion 
$\mathcal{E}'(\mathcal{P}_k)=\sum_{\mathcal{C}\in \mathcal{P}_k}P(\mathcal{C})V'(\mathcal{C})$
decreases from $V(P)$ to $0$ as the cardinality $k$ increases from $1$ to $n!$. If we introduce a stopping rule that consists of no longer splitting a node/cell $\mathcal{C}$ once its variability $V'(\mathcal{C})$ falls below a threshold $\epsilon\geq 0$, the partition $\mathcal{P}$ output is s.t. $\mathcal{E}'(\mathcal{P})\leq \epsilon$. Of course, the quantities $P(\mathcal{C})$ and $V'(\mathcal{C})$ are unknown in practice and an empirical version of the splitting criterion must be used, yielding the {\sc COAST} algorithm below.
\begin{figure}[!h]
\begin{center}
\fbox{
\begin{minipage}[t]{7.3cm}
\smallskip

\begin{center}
{\sc COAST Algorithm}
\end{center}

{\small

\begin{enumerate}
\item[]({\sc Input}) Ranking data $\Sigma_1,\; \ldots,\; \Sigma_N$, threshold $\epsilon\geq 0$, ranking aggregation algorithm $\mathcal{A}$
\item ({\sc Initialization.}) Start from the root node identified with the trivial partition $\mathcal{P}=\{\mathfrak{S}_n\}$ and set $\mathcal{J}_{\mathfrak{S}_n}=\{(i,j):1\leq i<j\leq n\}$.


\item ({\sc Iterations.}) While $\mathcal{P}(\epsilon):=\{\mathcal{C}\in \mathcal{P}: \; \widehat{V}'_N(\mathcal{C})> \epsilon\}\neq \emptyset$, do:


\item[(a)] \label{line:binary-splitting} ({\sc Binary splitting.}) For all $\mathcal{C}\in \mathcal{P}(\epsilon)$, find $(i,j)$ in the set of admissible pairs $\mathcal{J}_{\mathcal{C}}$ that minimize 
$$
\widehat{P}_N(\mathcal{C}_{i,j}^{(0)})\widehat{V}_N'(\mathcal{C}_{i,j}^{(0)})+\widehat{P}_N(\mathcal{C}_{i,j}^{(1)})\widehat{V}_N'(\mathcal{C}_{i,j}^{(1)})$$
and split $\mathcal{C}$ into $\mathcal{C}_{i,j}^{(0)}$ and $\mathcal{C}_{i,j}^{(1)}$.


\item[(b)] ({\sc Update.}) Redefine $\mathcal{P}$ as 
$$
\{\mathcal{C}\notin \mathcal{P}(\epsilon)\}\cup \{\mathcal{C}_{i,j}^{(m)}:\; (\mathcal{C},m)\in \mathcal{P}(\epsilon)\times\{0,1\} \}.
$$


\item ({\sc Local aggregation}) For all $\mathcal{C}\in \mathcal{P}$, apply algorithm $\mathcal{A}$ to $\{\Sigma_1,\; \ldots,\; \Sigma_N \}\cap \mathcal{C}$ yielding $\widehat{\sigma}^*_{\mathcal{C}}$.

\item[] ({\sc Output.}) Compute the empirical CRD
$$
\widehat{P}_{\mathcal{P}}=\sum_{\mathcal{C}\in \mathcal{P}}\widehat{P}_N(\mathcal{C})\delta_{\widehat{\sigma}^*_{\mathcal{C}}}.
$$
\end{enumerate}
}
\end{minipage}
}
\label{fig:trpseudo}
\caption{Pseudo-code for the {\sc COAST} algorithm.}
\end{center}
\vspace{-1cm}
\end{figure}

{\bf Pruning and $\crd$ model selection.} The partition $\mathcal{P}$ output by {\sc COAST}  is s.t. $\widehat{\mathcal{E}}'_N(\mathcal{P})\leq \epsilon$. The tuning parameter $\epsilon\geq 0$ governs the tree size and therefore the CRD's complexity. If $\epsilon \geq \widehat{V}_N'$, the tree is reduced to its root node $\mathcal{T}_1=\{\mathfrak{S}_n\}$ and the CRD estimate to a global (empirical) median, whereas, if $\epsilon=0$, the CRD estimate coincides with the raw empirical distribution $\widehat{P}_N$. Since it can be difficult to choose $\epsilon$, one strategy could be to select a 'small' value in order to first grow a 'large' tree, $\mathcal{T}_K$ with $K\leq N$ nodes say, then prune it, as follows. One starts from the 'large' tree $\mathcal{T}_K$ and form a sequence of subtrees  $\mathcal{T}_K \supset \mathcal{T}_{K-1}\supset \ldots \supset \mathcal{T}_1$, defining a sequence of partitions $\mathcal{P}'_K,\;\ldots,\; \mathcal{P}'_1$ (note that $\# \mathcal{P}'_k=k$, $\mathcal{P}'_K=\mathcal{P}_K$ and $\mathcal{P}'_1=\mathcal{P}_1$), by successively merging the internal nodes descending from the same parent, producing the smallest per-node increase in the variability $\widehat{\mathcal{E}}'_N$. One then selects the least complex subtree minimizing an estimate of the $\mathcal{E}'$ criterion. It can also be shown that this boils down to selecting the partition $\mathcal{P}'_k$ that minimizes the penalized criterion $\widehat{\mathcal{E}}'_N(\mathcal{P}'_k)+\lambda\cdot k $, for a specific value of $\lambda\geq 0$, a tuning parameter that governs the trade-off between empirical distortion and the number of cells (or local medians) retained for the approximation. See \cite{cart84} for further details.

\begin{figure*}[t]
    \centering

    \begin{subfigure}{.23\linewidth}
        \centering
        \includegraphics[width=\linewidth]{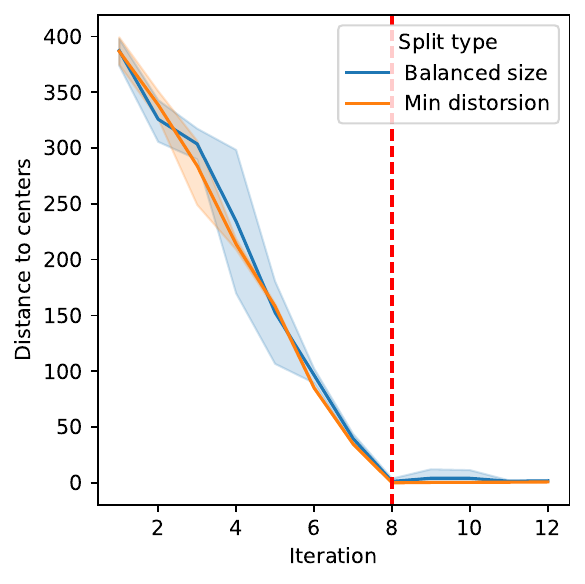}
        \caption{}
        \label{fig:b}
    \end{subfigure}
    \begin{subfigure}{.23\linewidth}
        \centering
        \includegraphics[width=\linewidth]{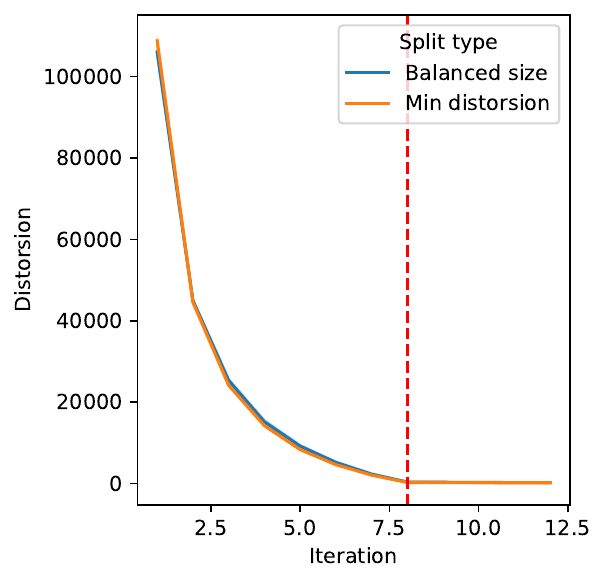}
        \caption{}
        \label{fig:c}
    \end{subfigure}
    \begin{subfigure}{.23\linewidth}
        \centering
        \includegraphics[width=\linewidth]{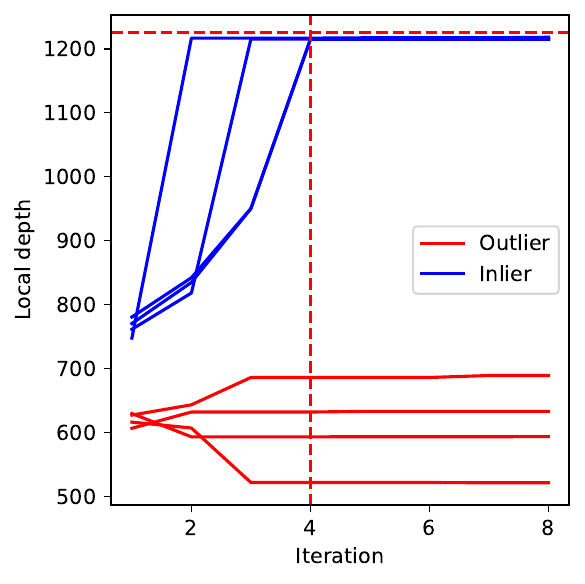}
        \caption{}
        \label{fig:a}
    \end{subfigure}
    \begin{subfigure}{.23\linewidth}
        \centering
        \includegraphics[width=\linewidth]{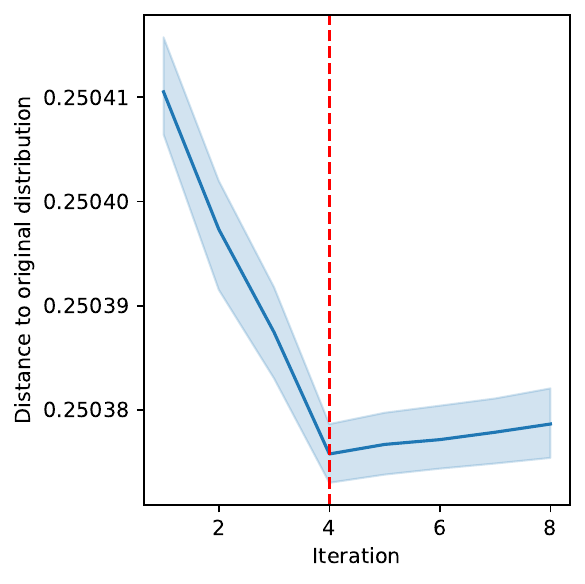}
        \caption{}
        \label{fig:d}
    \end{subfigure}

    \caption{Results for the Local Depth and Anomaly Detection~(\ref{fig:a}) and Mixtures of Mallows models~(\ref{fig:b},~\ref{fig:c},~\ref{fig:d}).}
    \vspace{-0.7cm}
\end{figure*}

\begin{remark}\label{thm:spl_rule}{\sc (Alternative splitting rule)} 
Another option for choosing the splitting pair $(i_{\mathcal{C}}, j_{\mathcal{C}})$ consists in selecting the pair $(i,j)$ s.t. $|p_{i,j}(\mathcal{C}) - 1/2|$ is closest to $0$, which also corresponds to the term which contributes most to the variability $V''(\mathcal{C})$ and leads to the most balanced child nodes possible. Compared to minimizing \eqref{eq:split_crit}, this strategy has computational advantages, the complexity decreasing from $O((\# \mathcal{I}(\mathcal{C}))^4)$ to $O(\# \mathcal{I}(\mathcal{C}))$. Although, as our experiments have shown, it often produces a comparable reduction in distortion/variability, but finding the conditions under which such splits are quasi-optimal remains an open question.  
\end{remark}



\subsection{Numerical Experiments and Applications}


Through several numerical experiments, we now illustrate the ability of the \textsc{COAST} algorithm to recover the modes of a multimodal ranking distribution and to produce an empirical $\crd$ offering a summary with much lower distortion than that provided that by a global ranking median. 
Rankings of $n=50$ items have been generated by means of a mixture of $K=8$ well-concentrated Mallows distributions (see the Appendix for a more detailed description). Figures ~\ref{fig:a} and~\ref{fig:b} show the average distance to the centers (\textit{i.e.} the medians of the Mallows components) of the local consensus rankings at each iteration of \textsc{COAST} and the decrease in the distortion bound $\widehat{\mathcal{E}}'$ for the two splitting rules, that related to the criterion \eqref{eq:split_crit} and that proposed in Remark~\ref{thm:spl_rule}. As empirically confirmed by the additional experiments discussed in the Appendix, The plateau revealed by the \textsc{COAST} algorithm allows the correct number of modes to be recovered. We now show how the output of the COAST algorithm can be used to effectively solve problems related to the statistical analysis of ranking data sampled from a multimodal distribution, such as anomaly detection or inference. 


\noindent {\bf Local depths \& anomaly detection.} 
The concept of \textit{ranking depth} $D_P$ recalled in subsection \ref{subsec:beyond} provides an efficient tool for anomaly detection in a ranking dataset when the inlier's law $P$ is unimodal, \textit{cf} \cite{pmlr-v151-goibert22a}. However, performance deteriorates significantly as soon as the distribution of 'normal' rankings becomes multimodal. In such situations, \textsc{COAST} enables us to rely on the \textit{local ranking depth} $D_{P_{\mathcal{C}}}(\sigma)$ to evaluate the degree of abnormality of a ranking $\sigma\in \mathcal{C}$, $\mathcal{C}\in \mathcal{P}$. As an illustration, consider the sequence of $\crd$'s learnt via \textsc{COAST} in $8$ iterations from a training ranking dataset drawn from a mixture of $K=4$ Mallows laws, as well as a test dataset composed of $4$ inlier rankings and 4 outliers, see the Appendix for more details. 
In Fig.~\ref{fig:d} the local depth of each ranking is plotted over the 8 iterations, clearly showing that the difference in local depth between inliers and outliers increases until the optimal number of partitions is reached. 
Other results leading to similar conclusions are presented in the Appendix. 


\noindent {\bf Ranking Distribution Inference.} As recalled in the Introduction, most ranking distribution inference techniques are based on fitting a rigid statistical model, see \cite{Marden96}, the 'nonparametric alternative' consisting in using the raw empirical distribution $\widehat{P}_N$ being inevitably inaccurate when the sample size $N$ is much smaller than $n!$, see subsection \ref{subsec:crd}. We finally show how to leverage \textsc{COAST}'s output when it comes to estimating a multimodal ranking distribution with accuracy. Note first that, up to an affine transform, the (local) Kendall $\tau$ depth can be naturally used to estimate a supposedly unimodal local distribution $P_{\mathcal{C}}$ via
\begin{equation}
\forall \sigma\in \mathcal{C},\;\;    \widetilde{P}_{\mathcal{C}}(\sigma) = \sum_{i<j} p_{\sigma^{-1}(i),\,\sigma^{-1}(j)}(\mathcal{C})/Z_{\mathcal{C}},
    \label{eq:smooth}
\end{equation}
where the complexity of computing the normalizing constant $Z_{\mathcal{C}}$ is $O(n^2 \#\mathcal{I}(\mathcal{C}))$. 
The approximation \eqref{eq:smooth} preserves certain essential properties.
They imply, for example, that the local median of the ranking distribution \eqref{eq:smooth}  can be recovered. 
See the Appendix for more details.
\textsc{COAST} permits to compute a statistical version of the approximant $\sum_{\mathcal{C}\in \mathcal{P}}P(\mathcal{C})\widetilde{P}_{\mathcal{C}}$. In Fig.~\ref{fig:d}, the $L_2$ distance between the latter and the original mixture is plotted, revealing that the distance reaches its minimimum when the number of mixture components $K=4$ is recovered.

\section{Conclusion and Perspectives}\label{sec:conclusion}

In this article, we introduced the new concept of consensus ranking distribution, thereby extending the notion of ranking median, particularly well suited to multimodal ranking distributions. We showed that the relevance of the statistical summary it provides is a delicate trade-off between dimension reduction and distortion minimization and proposed a tree-based learning algorithm for constructing CRDs from ranking data that achieve a satisfactory balance. Numerical results were presented to illustrate the effectiveness of the method and its usefulness for statistical ranking data analysis is demonstrated by several applications, which will hopefully pave the way for many more.
\newpage
\section*{Impact Statement}

This paper presents work whose goal is to advance the field of Machine Learning. There are many potential societal consequences of our work, none which we feel must be specifically highlighted here.

\bibliographystyle{icml2026}

\bibliography{biblio}

@inproceedings{sushi,
   author = {Toshihiro Kamishima},
   doi = {10.1145/956750.956823},
   booktitle = {Proceedings of the ACM SIGKDD International Conference on Knowledge Discovery and Data Mining},
   pages = {583-588},
   title = {Nantonac collaborative filtering: Recommendation based on order responses},
   year = {2003}
}

@inproceedings{himmi-etal-2024-towards,
    title = "Towards More Robust {NLP} System Evaluation: Handling Missing Scores in Benchmarks",
    author = "Himmi, Anas  and
      Irurozki, Ekhine  and
      Noiry, Nathan  and
      Cl{\'e}men{\c{c}}on, Stephan  and
      Colombo, Pierre",
    editor = "Al-Onaizan, Yaser  and
      Bansal, Mohit  and
      Chen, Yun-Nung",
    booktitle = "Findings of the Association for Computational Linguistics: EMNLP 2024",
    month = nov,
    year = "2024",
    address = "Miami, Florida, USA",
    publisher = "Association for Computational Linguistics",
    url = "https://aclanthology.org/2024.findings-emnlp.688/",
    doi = "10.18653/v1/2024.findings-emnlp.688",
    pages = "11759--11785"
}

@inproceedings{irurozki:hal-03972357,
  Title = {{Universal aggregation of permutations}},
  Author = {Irurozki, Ekhine and Cl{\'e}men{\c c}on, St{\'e}phan},
  url = {https://hal.science/hal-03972357},
  booktitle = {{Proceedings of DA2PL 2022: From Multiple-Criteria Decision Aid to Preference Learning}},
  year = {2022},
}

@article{LindermanMena2017reparameterizing,
	Author = {Linderman, Scott W. and Mena, Gonzalo E. and Cooper, Hal and Paninski, Liam and Cunningham, John P.},
	Journal = {arXiv preprint arXiv:1710.09508},
	Title = {Reparameterizing the {B}irkhoff Polytope for
	Variational Permutation Inference},
	Year = {2017}}

@inproceedings{clemenccon2010kantorovich,
	title={Kantorovich distances between rankings with applications to rank aggregation},
	author={Cl{\'e}men{\c{c}}on, St{\'e}phan and Jakubowicz, J{\'e}r{\'e}mie},
	booktitle={Joint European Conference on Machine Learning and Knowledge Discovery in Databases},
	pages={248--263},
	year={2010},
	organization={Springer}
}

@book{Lee90,author={A. J. Lee},year={1990},title={${U}$-statistics: Theory and practice},publisher={Marcel Dekker, Inc.},address={New York}}

@BOOK{cart84,
  author        = {L. Breiman and J. Friedman and R. Olshen and C. Stone},
  title         = {{Classification and Regression Trees}},
  publisher     = {Wadsworth and Brooks},
  year          = {1984},
}

@article{Liu1999,
  author    = {Regina Y. Liu and Jesse M. Parelius and Kesar Singh},
  title     = {Multivariate analysis by data depth: Descriptive statistics, graphics and inference},
  journal   = {The Annals of Statistics},
  volume    = {27},
  number    = {3},
  pages     = {783--858},
  year      = {1999},
  publisher = {Institute of Mathematical Statistics}
}

@InProceedings{pmlr-v151-goibert22a,
  title = 	 { Statistical Depth Functions for Ranking Distributions: Definitions, Statistical Learning and Applications },
  author =       {Goibert, Morgane and Cl\'emen\c{c}on, Stephan and Irurozki, Ekhine and Mozharovskyi, Pavlo},
  booktitle = 	 {Proceedings of The 25th International Conference on Artificial Intelligence and Statistics},
  pages = 	 {10376--10406},
  year = 	 {2022},
  editor = 	 {Camps-Valls, Gustau and Ruiz, Francisco J. R. and Valera, Isabel},
  volume = 	 {151},
  publisher =    {PMLR}}

@BOOK{Ser80,
  author =       {R.J. Serfling},
  title =        {Approximation theorems of mathematical statistics},
  publisher =    {John Wiley \& Sons},
  year =         {1980},
}

@InProceedings{KGD18,
  author    = {Korba, A. Garcia, F. D'Alch\'e'-Buc},
title={ A Structured Prediction Approach for Label Ranking},
booktitle={In Advances in Neural Information Processing Systems},
year={2018}
}

@InProceedings{CKS17,
  author    = {Korba, A. and Cl\'emen\c{c}on, S. and Sibony, E.},
  title     = {A Learning theory of ranking aggregation},
  booktitle = {Proceeding of AISTATS 2017},
  year      = {2017},
}

@InProceedings{kondor2010ranking,
  author    = {Kondor, R. and Barbosa, M. S.},
  title     = {Ranking with Kernels in {F}ourier space},
  booktitle = {The 23rd Conference on Learning Theory},
  year      = {2010},
  pages     = {451-463},
}

@InProceedings{pmlr-v37-sibony15,
  title = 	 {{MRA}-based Statistical Learning from Incomplete Rankings},
  author = 	 {Sibony, Eric and Cl\'emen\c{c}on, Stephan and Jakubowicz, Jérémie},
  booktitle = 	 {Proceedings of the 32nd International Conference on Machine Learning},
  pages = 	 {1432--1441},
  year = 	 {2015},
  editor = 	 {Bach, Francis and Blei, David},
  volume = 	 {37},
  publisher =    {PMLR}}

@inproceedings{achab2019bucket,
  title     = {Dimensionality Reduction and (Bucket) Ranking: a Mass Transportation Approach},
  author    = {Achab, Mastane and Korba, Anna and Cl{\'e}men\c{c}on, St{\'e}phan},
  booktitle = {Proceedings of the 30th International Conference on Algorithmic Learning Theory},
  volume    = {98},
  pages     = {1--30},
  year      = {2019},
  editor    = {Garivier, Aurélien and Kale, Satyen},
  publisher = {PMLR}
}

@inproceedings{goibert2023robust,
  title     = {Robust Consensus in Ranking Data Analysis: Definitions, Properties and Computational Issues},
  author    = {Goibert, Morgane and Calauz\`enes, Clément and Irurozki, Ekhine and Cl{\'e}men\c{c}on, Stephan},
  booktitle = {Proceedings of the 40th International Conference on Machine Learning},
  volume    = {202},
  pages     = {1--25},
  year      = {2023},
  publisher = {PMLR}
}

@InProceedings{procaccia2016optimal,
  author    = {Procaccia, A.D. and Shah, N.},
  title     = {Optimal Aggregation of Uncertain Preferences.},
  booktitle = {AAAI},
  year      = {2016},
  pages     = {608--614},
}

@Article{shah2015stochastically,
  author  = {Shah, N. B. and Balakrishnan, S. and Guntuboyina, A. and Wainright, M. J.},
  title   = {Stochastically Transitive Models for Pairwise Comparisons: Statistical and Computational Issues},
  journal = {arXiv preprint arXiv:1510.05610},
  year    = {2015},
}

@Article{JKSO16,
  author  = {Jang, M. and Kim, S. and Suh, C. and Oh, S.},
  title   = {Top-$ K $ Ranking from Pairwise Comparisons: When Spectral Ranking is Optimal},
  journal = {arXiv preprint},
  year    = {2016},
}

@InProceedings{dwork2001rank,
  author       = {Dwork, C. and Kumar, R. and Naor, M. and Sivakumar, D.},
  title        = {Rank aggregation methods for the web},
  booktitle    = {Proceedings of the 10th international conference on World Wide Web},
  year         = {2001},
  pages        = {613--622},
  organization = {ACM},
}

@Article{Kemeny59,
  author  = {Kemeny, J. G.},
  title   = {Mathematics without numbers},
  journal = {Daedalus},
  year    = {1959},
  volume  = {88},
  pages   = {571-591},
}

@Article{davidson1959experimental,
  author    = {Davidson, D. and Marschak, J.},
  title     = {Experimental tests of a stochastic decision theory},
  journal   = {Measurement: Definitions and theories},
  year      = {1959},
  volume    = {17},
  pages     = {274},
  publisher = {Citeseer},
}

@Article{boucheron2005theory,
  author    = {Boucheron, S. and Bousquet, O. and Lugosi, G.},
  title     = {Theory of classification: A survey of some recent advances},
  journal   = {ESAIM: probability and statistics},
  year      = {2005},
  volume    = {9},
  pages     = {323--375},
  publisher = {EDP Sciences},
}

@Article{Mallows57,
  author  = {Mallows, C. L.},
  title   = {Non-Null Ranking Models},
  journal = {Biometrika},
  year    = {1957},
  volume  = {44},
  number  = {1-2},
  pages   = {114--130},
}

@InProceedings{JKS16,
  author    = {Jiao, Y. and Korba, A. and Sibony, E.},
  title     = {Controlling the distance to a Kemeny consensus without computing it},
  booktitle = {Proceedings of ICML 2016},
  year      = {2016},
}

@Book{Vapnik,
  title     = {{The Nature of Statistical Learning Theory}},
  publisher = {Springer},
  year      = {2000},
  author    = {Vapnik, V. N.},
  series    = {Lecture Notes in Statistics},
}

@Book{Deza,
  title     = {{Encyclopedia of Distances}},
  publisher = {Springer},
  year      = {2009},
  author    = {M.M. Deza and E. Deza},
}

@article{Mao2008,
   author = {Yi Mao and Guy Lebanon},
   journal = {Journal of Machine Learning Research},
   pages = {2401-2429},
   title = {Non-Parametric Modeling of Partially Ranked Data},
   volume = {9},
   year = {2008},
}

@Article{shah2015simple,
  author  = {Shah, N.B. and Wainwright, M.J.},
  title   = {Simple, robust and optimal ranking from pairwise comparisons},
  journal = {arXiv preprint arXiv:1512.08949},
  year    = {2015},
}

@Article{fishburn1973binary,
  author    = {Fishburn, P. C.},
  title     = {Binary choice probabilities: on the varieties of stochastic transitivity},
  journal   = {Journal of Mathematical psychology},
  year      = {1973},
  volume    = {10},
  number    = {4},
  pages     = {327--352},
  publisher = {Elsevier},
}

@book{Marden96,
	author = {Marden, J. I.},
	title = {Analyzing and Modeling Rank Data},
	publisher = {CRC Press},
	address = {London},
	year = {1996},
}

@Article{Hudry08,
  author  = {Hudry, O.},
  title   = {{NP}-hardness results for the aggregation of linear orders into median orders},
  journal = {Ann. Oper. Res.},
  year    = {2008},
  volume  = {163},
  pages   = {63--88},
}

@Article{Plackett75,
  author  = {Plackett, R. L.},
  title   = {The analysis of permutations},
  journal = {Applied Statistics},
  year    = {1975},
  volume  = {2},
  number  = {24},
  pages   = {193--202},
}

@InProceedings{KB05,
  author    = {Koltchinskii, V. and Beznosova, O.},
  title     = {Exponential Convergence Rates in Classification},
  booktitle = {Proceedings of COLT 2005},
  year      = {2005},
}

\newpage
\appendix
\onecolumn
\section{Technical Proofs}

\subsection*{Proof of Proposition \ref{prop:gen_bound}}
Observe first that, with probability one,
\begin{equation}\label{eq:dec}
W_d(P,P_{\widehat{P}_N})/2\leq \mathcal{E}'(\widehat{\mathcal{P}}_N)\leq \min_{\mathcal{P}\in \Pi_K}\mathcal{E}'(\mathcal{P})+2\max_{\mathcal{P}\in \Pi_K}\left\vert \widetilde{\mathcal{E}}'_N(\mathcal{P})- \mathcal{E}'(\mathcal{P}) \right\vert + 2\max_{\mathcal{P}\in \Pi_K}\left\vert \widetilde{\mathcal{E}}'_N(\mathcal{P})- \widehat{\mathcal{E}}'_N(\mathcal{P}) \right\vert.
\end{equation}

\begin{lemma}
For any partition $\mathcal{P}$ of $\mathfrak{S}_n$ s.t. $P(\mathcal{C})>0$ for all $\mathcal{C}\in \mathcal{P}$, we have: $\forall t>0$:
\begin{equation}\label{eq:bound_Ustat}
\mathbb{P}\left\{ \left\vert \widehat{\mathcal{E}}_N'(\mathcal{P})  - \mathcal{E}'(\mathcal{P}) \right\vert \geq t\right\}\leq 2e^{-Nt^2/\kappa^2_{\mathcal{P}}},
\end{equation}
where $\kappa_{\mathcal{P}}:=\max_{\mathcal{C}\in \mathcal{P}}\left\{\max_{(\sigma,\sigma')\in \mathcal{C}^2}d(\sigma,\sigma')/P(\mathcal{C}) \right\}$.
\end{lemma}
\begin{proof}
The bound \eqref{eq:bound_Ustat} is obtained directly by applying the version of Hoeffding's exponential deviation inequality for $U$-statistics (see Theorem A of section 5.6 in \cite{Ser80}) to $\widehat{\mathcal{E}}_N'(\mathcal{P})$, a $U$-statistic of degree $2$ whose symmetric kernel satisfies: $\forall (\sigma,\sigma')\in \mathfrak{S}_n^2$, $0\leq \Phi_{\mathcal{P}}(\sigma,\sigma')\leq \kappa_{\mathcal{P}}$.
\end{proof}
As $\kappa_{\mathcal{P}}\leq \gamma(\Pi_K)/\chi(\Pi_K)$ for all $\mathcal{P}\in \Pi_K$, the union bound combined with \eqref{eq:bound_Ustat} implies that, for all $\delta\in (0,1)$, we have:
\begin{equation}\label{eq:term2}
\max_{\mathcal{P}\in \Pi_K}\left\vert \widetilde{\mathcal{E}}'_N(\mathcal{P})- \mathcal{E}'(\mathcal{P}) \right\vert\leq \sqrt{\frac{\gamma(\Pi_K)}{\chi(\Pi_K)}\frac{\log(2\# \Pi_K/\delta)}{N}},
\end{equation}
with probability larger than $1-\delta$.

Next, let us turn to the third term on the right hand side of \eqref{eq:dec}. Combining Hoefdding's inequality with the union bound, we immediately obtain that the event $\cap_{\mathcal{C}\in \mathcal{P}}\{N_{\mathcal{C}}\geq 2\}$ occurs with overwhelming probability:
\begin{equation}\label{eq:event}
\mathbb{P}\left\{ \cap_{\mathcal{C}\in \mathcal{P}}\{N_{\mathcal{C}}\geq 2\} \right\} \geq 1-Ke^{-2(N\chi(\Pi_K)-1)^2},
\end{equation}
as soon as $N\geq 1/\chi(\Pi_K)$.
On this event, notice that, for all $\mathcal{P}\in \Pi_K$, we have:
\begin{multline}\label{eq:bound_dec1}
\left\vert \widetilde{\mathcal{E}}'_N(\mathcal{P})- \widehat{\mathcal{E}}'_N(\mathcal{P})\right\vert =\left\vert \frac{2}{N(N-1)}\sum_{1\leq i<j\leq N}\sum_{\mathcal{C}\in \mathcal{P}}\left(\frac{N-1}{N_{\mathcal{C}}-1}  -\frac{1}{P(\mathcal{C})}\right)\mathbb{I}\{(\Sigma_i,\Sigma_j)\in \mathcal{C}^2\}d(\Sigma_i,\Sigma_j)\right\vert \\
\leq K \vert\vert d\vert\vert_{\infty} \cdot \max_{\mathcal{C}\in \mathcal{P}}\left\vert \frac{N-1}{N_{\mathcal{C}}-1}  -\frac{1}{P(\mathcal{C})} \right\vert .
\end{multline}
Combining Hoefdding's inequality with the union bound, we immediately obtain that, for all $\delta\in (0,1)$, we have with probability at least $1-\delta$:
\begin{equation}\label{eq:Hoeffding}
\max_{\mathcal{C}\in \mathcal{P}}\left\vert \widehat{P}_N(\mathcal{C})-P(\mathcal{C}) \right\vert \leq \sqrt{\frac{\log(2K/\delta)}{2N}},
\end{equation}
where $\widehat{P}_N(\mathcal{C})=N_{\mathcal{C}}/N$ for all $\mathcal{C}\in \mathcal{P}$.
Observing that $\vert(N_{\mathcal{C}}-1)/(N-1)-\widehat{P}_N(\mathcal{C})\vert \leq 2/(N-1)$ with probability one and combining the union bound with \eqref{eq:event}, \eqref{eq:bound_dec1}, \eqref{eq:Hoeffding} and the Taylor expansion
\begin{equation*}
    \frac{1}{x} = \frac{1}{a} - \frac{(x-a)}{a^2} + \frac{(x-a)^2}{xa^2},
\end{equation*}
one obtains that, for all $\delta\in (0,1)$, we have with probability larger than $1-\delta$:
\begin{equation}\label{eq:term3}
\left\vert \widetilde{\mathcal{E}}'_N(\mathcal{P})- \widehat{\mathcal{E}}'_N(\mathcal{P})\right\vert \leq C\frac{K\vert\vert d\vert\vert_{\infty}}{\chi^2(\Pi_K)}\sqrt{\frac{\log(4K/\delta)}{2N}},
\end{equation}
where $C$ is a universal constant, as soon as $N\geq \max\{\log(4K/\delta),\; (1+\sqrt{\log(2K/\delta)/2})/\chi(\Pi_K)\}$.
Finally, combining \eqref{eq:dec}, \eqref{eq:term2} and \eqref{eq:term3} with the union bound yiels the desired result.

\subsection*{Statistical Versions of Optimal CRDs}
Let $1\leq K\leq n!$ and $\Pi_K$ be a collection of partitions of $\mathfrak{S}_n$ counting $K$ cells. Consider now the distribution 
\begin{equation}\label{eq:emp_crd}
\widehat{P}_{\widehat{\mathcal{P}}_N}:=\sum_{\mathcal{C}\in \widehat{\mathcal{P}}_N}\widehat{P}_N(\mathcal{C})\delta_{\widehat{\sigma}^*_{\mathcal{C}}},
\end{equation}
where $\widehat{\mathcal{P}}_N$ is a minimizer of \eqref{eq:emp_crit} over $\Pi_K$ and, for all $\mathcal{C}\in \widehat{\mathcal{P}}_N$,
\begin{equation}\label{eq:emp_local_median}
\widehat{\sigma}^*_{\mathcal{C}}\in \argmin_{\sigma\in \mathfrak{S}_n}\sum_{i:\; \Sigma_i \in \mathcal{C}}d(\Sigma_i,\sigma). 
\end{equation}
Conditioned upon the $\Sigma_i$'s, we introduce as well the ranking distribution
\begin{equation}\label{eq:emp_crd2}
\widetilde{P}_{\widehat{\mathcal{P}}_N}:=\sum_{\mathcal{C}\in \widehat{\mathcal{P}}_N}P(\mathcal{C})\delta_{\widehat{\sigma}^*_{\mathcal{C}}},
\end{equation}

By triangle inequality, we have:
\begin{equation}\label{eq:triangle}
W_d(P, \widehat{P}_{\widehat{\mathcal{P}}_N})\leq W_d(\widetilde{P}_{\widehat{\mathcal{P}}_N}, \widehat{P}_{\widehat{\mathcal{P}}_N})+ W_d(P, \widetilde{P}_{\widehat{\mathcal{P}}_N}).
\end{equation}
Conditioned upon the $\Sigma_i$'s, consider the coupling $(\widetilde{\Sigma}, \Sigma)$ of the pair of ranking distributions $(\widetilde{P}_{\widehat{\mathcal{P}}_N}, P)$ defined by: 
\begin{equation}
\forall \mathcal{C}\in \widehat{\mathcal{P}}_N,\;\;  \mathbb{P}\{\widetilde{\Sigma}\in \mathcal{C} \mid \Sigma \in \mathcal{C},\; \Sigma_1,\; \ldots,\; \Sigma_N\}=+1 \text{ almost-surely}.
\end{equation}
In order to lighten the notation, we omit to write explicitely the conditioning upon the $\Sigma_i$'s from now on. We have:
\begin{multline}\label{eq:term2}
W_d(P, \widetilde{P}_{\widehat{\mathcal{P}}_N})\leq \mathbb{E}[d(\widetilde{\Sigma}, \Sigma)]=\sum_{\mathcal{C}\in \widehat{\mathcal{P}}_N}P(\mathcal{C})\mathbb{E}[d(\Sigma, \widehat{\sigma}^*_{\mathcal{C}})\mid \Sigma\in \mathcal{C}]\leq \\ \mathcal{E}(\widehat{\mathcal{P}}_N)+\sum_{\mathcal{C}\in \widehat{\mathcal{P}}_N}P(\mathcal{C})\left( \mathbb{E}[d(\Sigma, \widehat{\sigma}^*_{\mathcal{C}})\mid \Sigma\in \mathcal{C}]-\mathbb{E}[d(\Sigma, \sigma^*_{\mathcal{C}})\mid \Sigma\in \mathcal{C}]\right)\leq \\2\mathcal{E}'(\widehat{\mathcal{P}}_N)+
\max_{\mathcal{P}\in \Pi_K}\max_{\mathcal{C}\in \mathcal{P}}\left( \mathbb{E}[d(\Sigma, \widehat{\sigma}^*_{\mathcal{C}})\mid \Sigma\in \mathcal{C}] -\mathbb{E}[d(\Sigma, \sigma^*_{\mathcal{C}})\mid \Sigma\in \mathcal{C}] \right).
\end{multline}
Each expectation $\mathbb{E}[d(\Sigma, \widehat{\sigma}^*_{\mathcal{C}})\mid \Sigma\in \mathcal{C}]$ involved in the summation on the right hand side of \eqref{eq:term2} can be bounded by means of the lemma below, which corresponds to a version of Proposition 9 in \cite{CKS17} where the number of observations available to compute an empirical median of $P_{\mathcal{C}}$ is random, equal to $N_{\mathcal{C}}$, which is equivalent to $NP(C)$ with probability one.

\begin{lemma}\label{lem:local_emp_consensus}
Let $\mathcal{C}\subset \mathfrak{S}_n$ s.t. $P(\mathcal{C})\geq \chi>0$. Let $\widehat{\sigma}^*_{\mathcal{C}}$ be defined by \eqref{eq:emp_local_median}. Then, for all $\delta\in (0,1)$, we have with probability larger than $1-\delta$: 
\begin{equation}
\mathbb{E}[d(\Sigma, \widehat{\sigma}^*_{\mathcal{C}})\mid \Sigma\in \mathcal{C}]-\mathbb{E}[d(\Sigma, \sigma^*_{\mathcal{C}})\mid \Sigma\in \mathcal{C}]\leq \frac{n(n-1)}{2}\sqrt{\frac{2\log(2n(n-1)/\delta)}{(1+\chi)N}},
\end{equation}
as soon as $N\geq \log(2/\delta)/2$.
\end{lemma}

\begin{proof}
It follows from Hoeffding's inequality that, with probability at least $1-\delta$:
\begin{equation*}\label{eq:event1}
\widehat{P}_N(\mathcal{C})\geq P(\mathcal{C})+\sqrt{\frac{\log(1/\delta)}{2N}}.
\end{equation*}
Hence, with probability larger than $1-\delta$, we have:
\begin{equation*}\label{eq:event1}
N_{\mathcal{C}}\geq N(\chi+1),
\end{equation*}
as soon as $N\geq \log(1/\delta)/2$.
Combining this with the bound stated in assertion (ii) of Proposition 9 in \cite{CKS17} (applied conditionally to $N_{\mathcal{C}}$) and the union bound yields the desired result.
\end{proof}

Applying Lemma \ref{lem:local_emp_consensus} combined with the union bound, we obtain that, with probability at least $1-\delta$:
\begin{equation}\label{eq:term_max}
\max_{\mathcal{P}\in \Pi_K}\max_{\mathcal{C}\in \mathcal{P}}\left( \mathbb{E}[d(\Sigma, \widehat{\sigma}^*_{\mathcal{C}})\mid \Sigma\in \mathcal{C}] -\mathbb{E}[d(\Sigma, \sigma^*_{\mathcal{C}})\mid \Sigma\in \mathcal{C}] \right)\leq \frac{n(n-1)}{2}\sqrt{\frac{2\log(2n(n-1)K\#\Pi_K/\delta)}{(1+\chi(\Pi_K))N}},
\end{equation}
as soon as $N\geq \log(2K\# \Pi_K/\delta)/2$.

Now turn to the first term on the right hand side of \eqref{eq:triangle}. Conditioned upon the $\Sigma_i$'s, consider any coupling 
$(\widetilde{\Sigma}, \widehat{\Sigma})$ of the pair of ranking distributions $(\widetilde{P}_{\widehat{\mathcal{P}}_N}, \widetilde{P}_{\widehat{\mathcal{P}}_N})$ such that: $\forall \mathcal{C}\in \widehat{\mathcal{P}}_N$,
\begin{equation*}
\begin{array}{cccc}
\mathbb{P}\{\widehat{\Sigma}\in \mathcal{C} \mid \widetilde{\Sigma} \in \mathcal{C},\; \Sigma_1,\; \ldots,\; \Sigma_N\}&=+1& \text{ almost-surely}& \text{if } P(\mathcal{C})\leq \widehat{P}_N(\mathcal{C}),\\
\mathbb{P}\{\widetilde{\Sigma}\in \mathcal{C} \mid \widehat{\Sigma} \in \mathcal{C},\; \Sigma_1,\; \ldots,\; \Sigma_N\}&=+1& \text{ almost-surely}& \text{if } P(\mathcal{C})\geq \widehat{P}_N(\mathcal{C}).
\end{array}
\end{equation*}
We clearly have:
\begin{equation}
 W_d(\widetilde{P}_{\widehat{\mathcal{P}}_N}, \widehat{P}_{\widehat{\mathcal{P}}_N})\leq \mathbb{E}[d(\widetilde{\Sigma}, \widehat{\Sigma})]\leq \vert\vert d \vert\vert_{\infty}\sum_{\mathcal{C}\in \widehat{\mathcal{P}}_N}\left\vert \widehat{P}_N(\mathcal{C})-P(\mathcal{C}) \right\vert\leq \vert\vert d \vert\vert_{\infty}K \max_{\mathcal{P}\in \Pi_K}\max_{\mathcal{C}\in \mathcal{P}}\left\vert \widehat{P}_N(\mathcal{C})-P(\mathcal{C}) \right\vert.
\end{equation}
Combined with \eqref{eq:Hoeffding} and the union bound, the bound above permits to show that, with probability larger than $1-\delta$:
\begin{equation}\label{eq:term_last}
 W_d(\widetilde{P}_{\widehat{\mathcal{P}}_N}, \widehat{P}_{\widehat{\mathcal{P}}_N})\leq \vert\vert d \vert\vert_{\infty}K \sqrt{\frac{\log(2K\#\Pi_K/\delta)}{2N}}.
\end{equation}

From Proposition \ref{prop:gen_bound} combined with \eqref{eq:triangle}, \eqref{eq:term2}, \eqref{eq:term_max} and \eqref{eq:term_last} we deduce the following result.

\begin{proposition}\label{prop:gen_bound2}
Suppose the assumptions of Proposition \ref{prop:gen_bound} are satisfied. Consider any minimizer $\widehat{\mathcal{P}}_N$ of the empirical criterion \eqref{eq:emp_crit} over $\Pi_K$. For all $\delta\in (0,1)$, we have with probability at least $1-\delta$:
\begin{multline*}
W_d(P, \widehat{P}_{\widehat{\mathcal{P}}_N})/2\leq \mathcal{E}'(\widehat{\mathcal{P}}_N)\leq \min_{\mathcal{P}\in \Pi_K}\mathcal{E}'(\mathcal{P})+
 2\sqrt{\frac{\gamma(\Pi_K)}{\chi(\Pi_K)}\frac{\log(12\# \Pi_K/\delta)}{N}}+ 2C\frac{K\vert\vert d\vert\vert_{\infty}}{\chi^2(\Pi_K)}\sqrt{\frac{\log(24K/\delta)}{2N}}\\
 \frac{1}{2}\binom{n}{2}\sqrt{\frac{6\log(2n(n-1)K\#\Pi_K/\delta)}{(1+\chi(\Pi_K))N}}+
\frac{\vert\vert d \vert\vert_{\infty}K}{2} \sqrt{\frac{\log(6K\#\Pi_K/\delta)}{2N}},
\end{multline*}
when $N\geq \max\{\log(24K/\delta),\; (1+\sqrt{\log(12K/\delta)/2})/\chi(\Pi_K),\; \log(6K\# \Pi_K/\delta)/2\}$.
\end{proposition}




\subsection{Ranking Distribution Characterization by Local Pairwise Marginals}

While it is well-known that the pairwise probabilities $p_{i,j}$ do not characterize a ranking distribution $P$ in general, the collection of local pairwise probabilities $p_{i,j}(\mathcal{C})$, where $\mathcal{C}$ describes the ensemble of ranking subsets of the form \eqref{eq:cell_form}. Consider the lexicographic order on $\{(i,j)\in \n^2:\; i<j \}$: $(i,j)\prec (k,l)$ iff either '$i<k$' or else '$i=k$ and $j<l$'. 

Equipped with this notation, for any $\sigma\in \mathfrak{S}_n$, one can write
\begin{multline}
P(\sigma)=\mathbb{P}\left\{\bigcap_{i<j}\{(\Sigma(i)-\Sigma(j))(\sigma(i)-\sigma(j))>0\}\right\}\\=\left(p_{1,2}\mathbb{I}\{\sigma(1)<\sigma(2)\}+(1-p_{1,2})\mathbb{I}\{\sigma(2)<\sigma(1)\}\right)\times \\ \mathbb{P}\left\{\bigcap_{k<l:\; (1,2)\prec (k,l)}\{(\Sigma(k)-\Sigma(l))(\sigma(l)-\sigma(l))>0\}\mid (\Sigma(1)-\Sigma(2))(\sigma(1)-\sigma(2))>0 \right\}
\end{multline}
and the result follows by induction.

The assumption underlying the \textsc{COAST} algorithm is that it is possible to approximate the distribution $P$ under study based on knowledge of the local pairwise probabilities $p_{i,j}(\mathcal{C})$, for a few well-chosen subsets $\mathcal{C}$ of the form \eqref{eq:cell_form} forming a partition of $\mathfrak{S}_n$.

\subsection{Alternative Dispersion Criterion}

The dispersion of a ranking probability distribution $P$ can be measured by
$$
V''_P=\sum_{k<l}\min\{p_{k,l},\; 1-p_{k,l}\},
$$
which quantity can be directly computed from the pairwise probabilities $p_{i,j}$, as $V'_P$. Notice that $V_P\leq V'_P\leq V''_P$ and the equalities hold true when $P$ is SST. Observe also that, for any partition $\mathcal{P}$ of $\mathfrak{S}_n$, we have:
\begin{equation}
\sum_{\mathcal{C}\in \mathcal{P}}P(\mathcal{C})V''(\mathcal{C})\leq V''_{P}.
\end{equation}
This simply results from the fact that: $\forall \mathcal{C}\in \mathcal{P}$,
$$
V''(\mathcal{C})=\sum_{k<l}\min\{p_{k,l}(\mathcal{C}),\; 1-p_{k,l}(\mathcal{C})\}
$$
and, for all $k<l$,
\begin{equation}
p_{k,l}=\sum_{\mathcal{C}\in \mathcal{P}}P(\mathcal{C})p_{k,l}(\mathcal{C}).
\end{equation}



\section{Additional Experiments and Applications}

\subsection{Visualization beyond ranking Depth: DD-plots on local vs. global depth}
Visualizing distributions of rankings is particularly challenging, since standard tools such as histograms are not applicable. 
Originally introduced for multivariate data in \cite{Liu1999} and \cite{pmlr-v151-goibert22a}, DD-plots were proposed as a visualization tool to illustrate variations in location and scale across different distributions. We extend this framework to the setting of CRD and general mixtures of rankings. Our goal is to visualize a CRD through their relative positions with respect to both the partition-wise (local) depth and the overall (global) depth, denoted by $D_{\hat P(\mathcal C)}$ and $D_{\hat P_N}$, respectively.

 Given a sample $S$ drawn from a CDR model, we consider the associated empirical measures $\hat P_{\mathcal C}$, corresponding to a partition $\mathcal C$, and $\hat P_N$, associated with the entire sample. The DD-plot is constructed by representing each ranking $\sigma \in S$ as a point in the plane with coordinates
\[
\bigl( D_{\hat P(\mathcal C)}(\sigma),\, D_{\hat P_N}(\sigma) \bigr).
\]
Points are color-coded according to the mixture component from which the ranking was generated. 
This representation enables a direct visual comparison between local and global depth behaviors, thereby revealing structural features of the CRD, such as separation, overlap, or heterogeneity among components, that may not be apparent when relying on a single notion of depth.

In  Figure~\ref{fig:sampling_10_4}, we illustrate the DD-plot obtained from samples  drawn from a mixture of permutations of size $n = 10$ with $K = 4$ components, ranging from a highly peaked distribution (left panel) to a more uniform distribution (right panel). After $K$ iterations of the COAST algorithm the partitioning structure recovered allows the evaluation of $( D_{\hat P(\mathcal C)}(\sigma),\, D_{\hat P_N}(\sigma) )$. The chosen local depth $D_{\hat P(\mathcal C)}(\sigma)$ corresponds to cluster $k=0$. 

This graphical representation reveals a clear clustering structure: points tend to aggregate in specific regions of the DD-plot, reflecting the underlying clustering of the components in the data. This highlights the diagnostic power of the DD-plot for identifying structural heterogeneity in ranking samples. In particular, while observations are not clearly distinguishable when considering the global depth $D_{\hat P_N}$, they become well separated when local depth values $D_{\hat P(\mathcal C)}$ are used.

The same conclusion are drawn for larger number of components, Figure~\ref{fig:sampling_10_8}, and Plackett-Luce model, Figure~\ref{fig:sampling_PL}.

\begin{figure}[h]
\centering
\includegraphics[width=.2\linewidth]{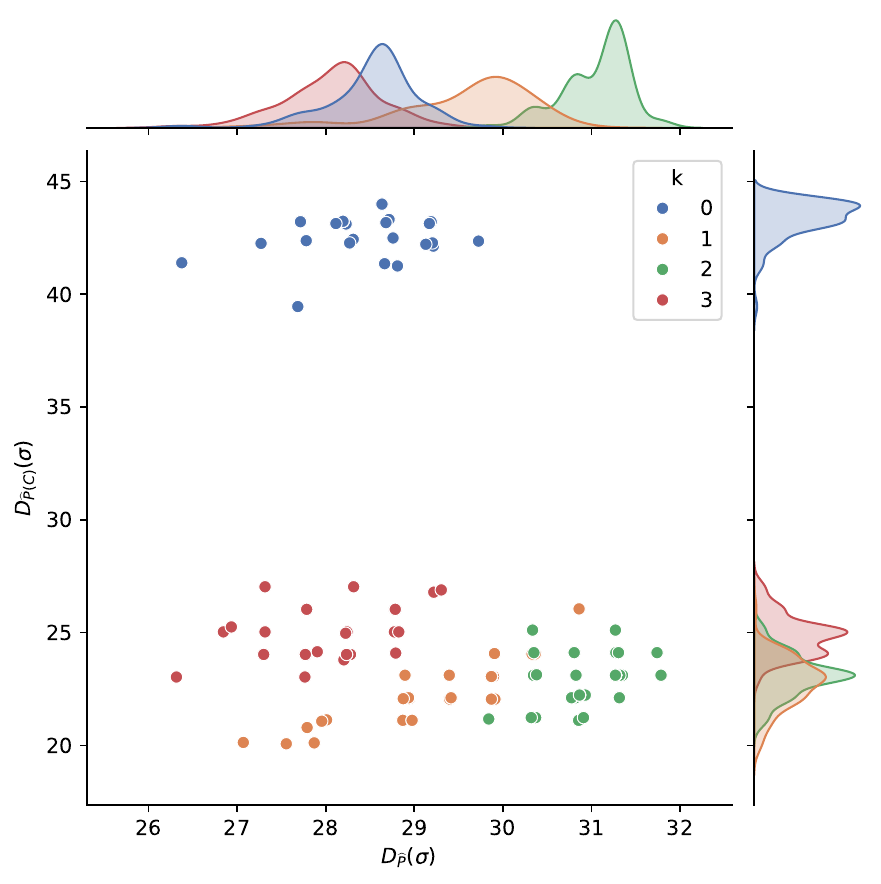}
\includegraphics[width=.2\linewidth]{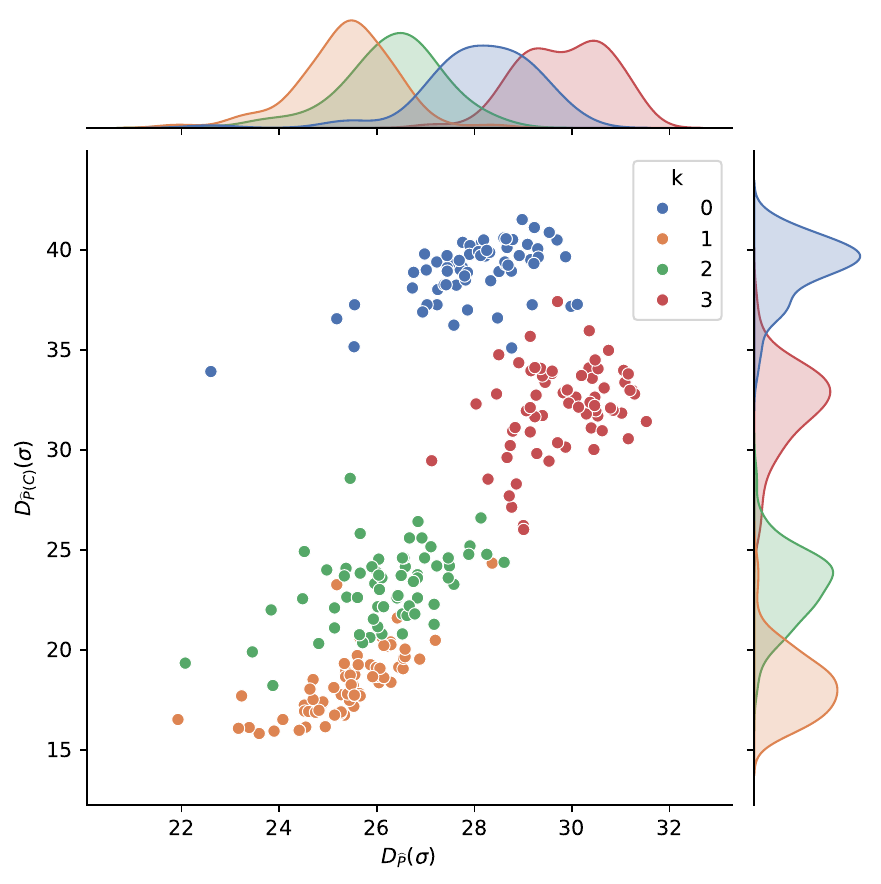}
\includegraphics[width=.2\linewidth]{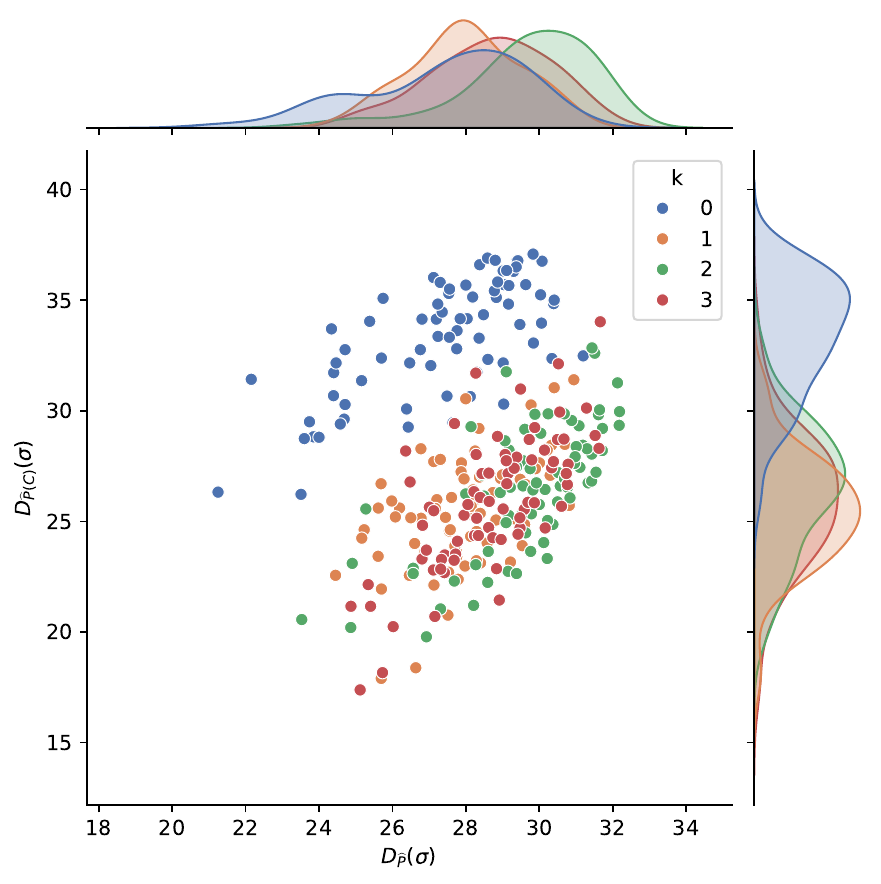}
\caption{DD-plots on local vs. global depth: Mallows model $n=10, k=4$}
\label{fig:sampling_10_4}
\end{figure}

\begin{figure}[h]
\centering
\includegraphics[width=.2\linewidth]{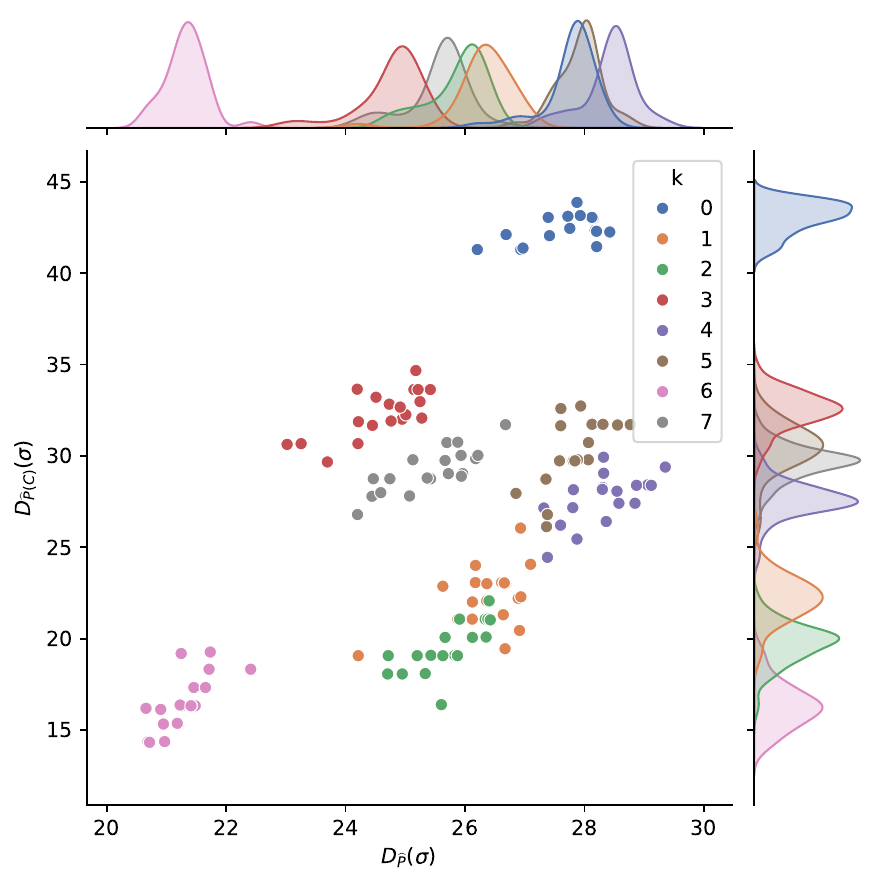}
\includegraphics[width=.2\linewidth]{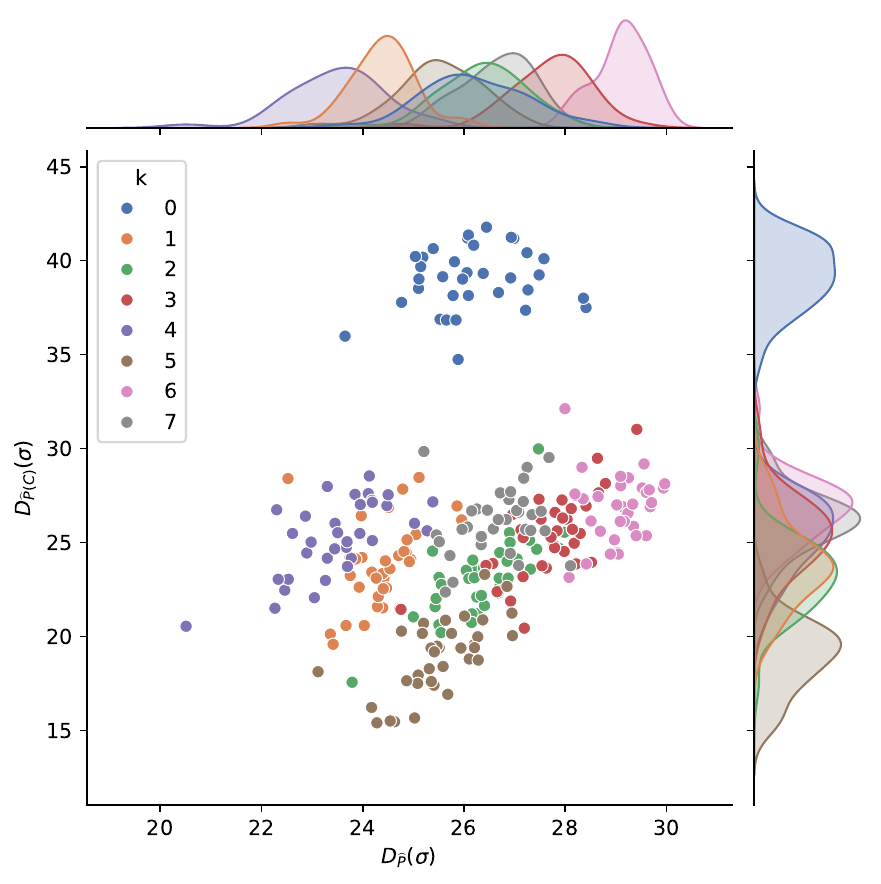}
\includegraphics[width=.2\linewidth]{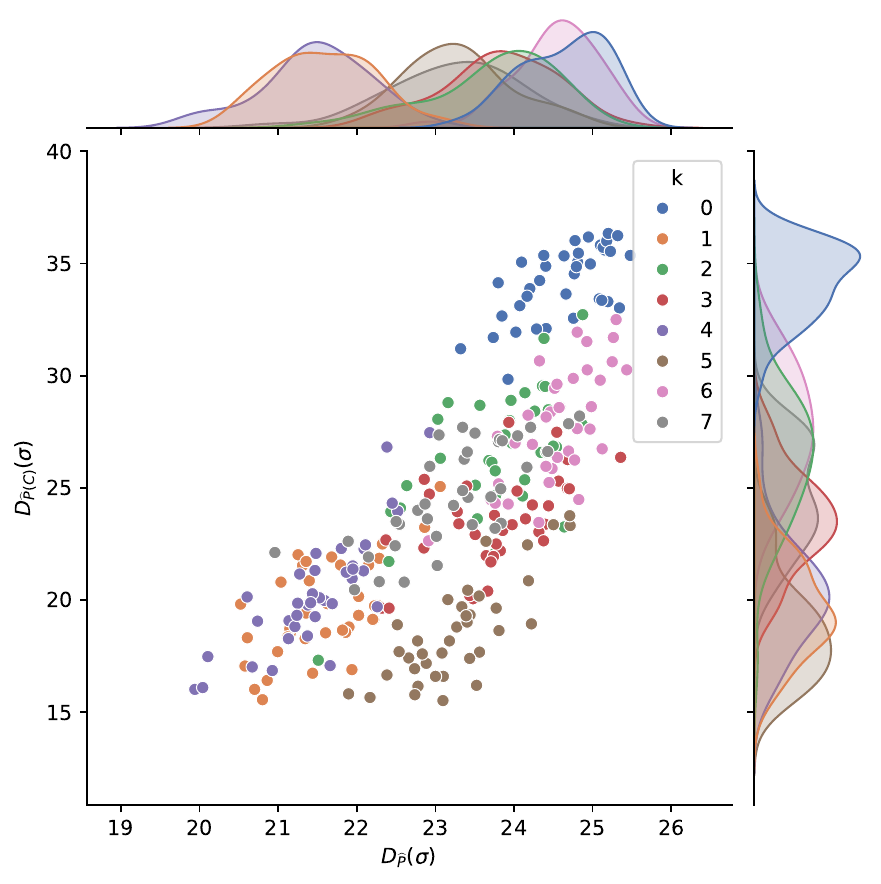}
\caption{DD-plots on local vs. global depth: Mallows model $n=10, k=8$}
\label{fig:sampling_10_8}
\end{figure}

\begin{figure}[h]
\centering
\includegraphics[width=.2\linewidth]{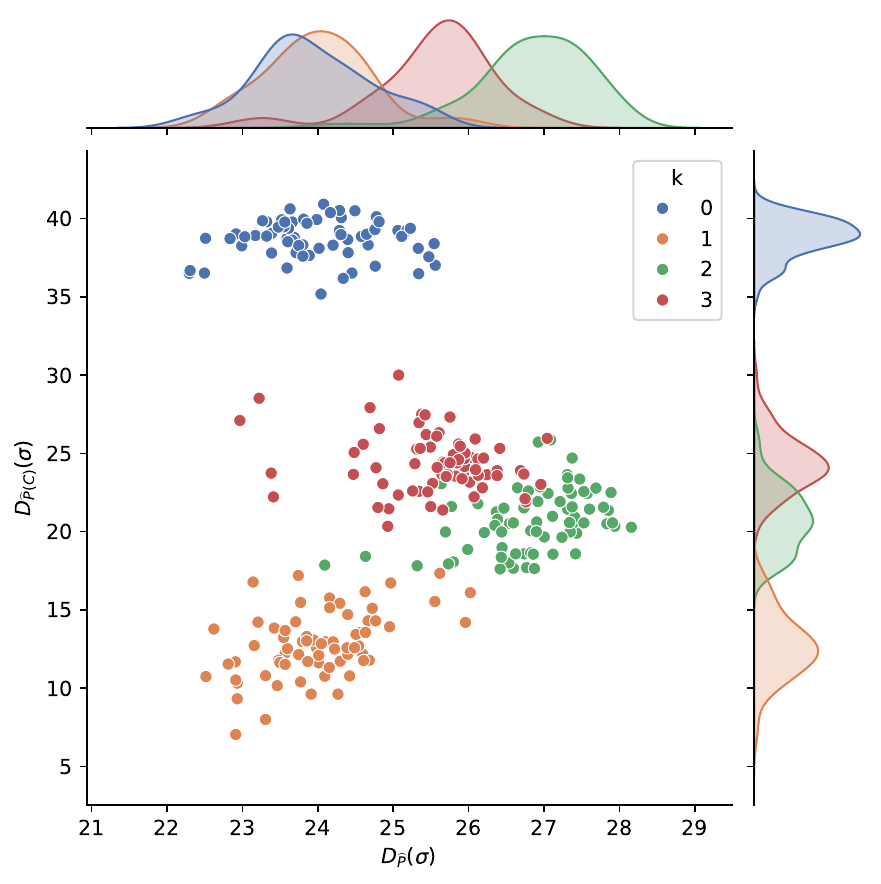}
\includegraphics[width=.2\linewidth]{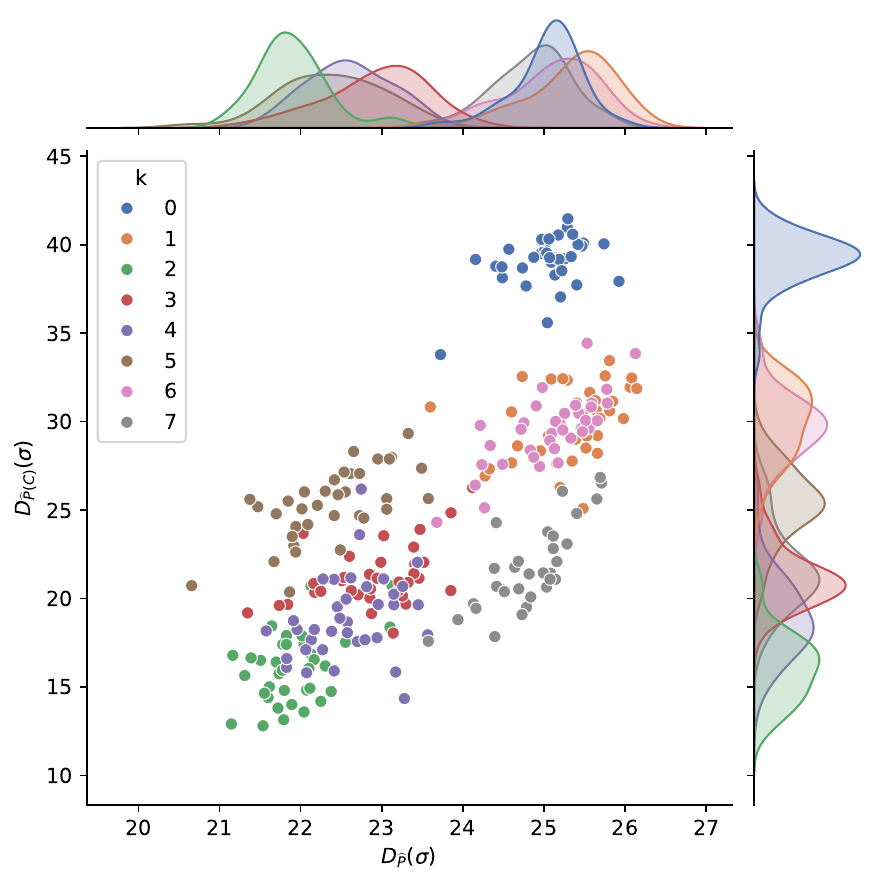}
\caption{DD-plots on local vs. global depth: Plackett-Luce for $n=10$, $k=4$ and $n=10$, $k=8$.}
\label{fig:sampling_PL}
\end{figure}

\subsection{Further experiments on the COAST}

The experimental framework of the paper is detailed in this Appendix, where we consider a broader range of configurations and models, including the Plackett–Luce and Mallows models. In all experiments, the mixture components share the same scale parameter. Two splitting criteria are evaluated, and both yield results of comparable quality. Although the alternative criterion is not optimal in the sense that it does not guarantee the largest decrease of the distortion at each iteration, both procedures ultimately converge to a similar level of distortion.

In addition, we report a table summarizing the computational times (in seconds) obtained on a standard personal computer. The theoretical differences in computational complexity are reflected in practice: time results are similar for small $n$ and $k$; however, for $n = 150$ the exact splitting method requires more than 15 minutes of computation, whereas the alternative criterion completes in approximately 15 seconds.

\begin{figure}[h]
\centering
\includegraphics[width=.2\linewidth]{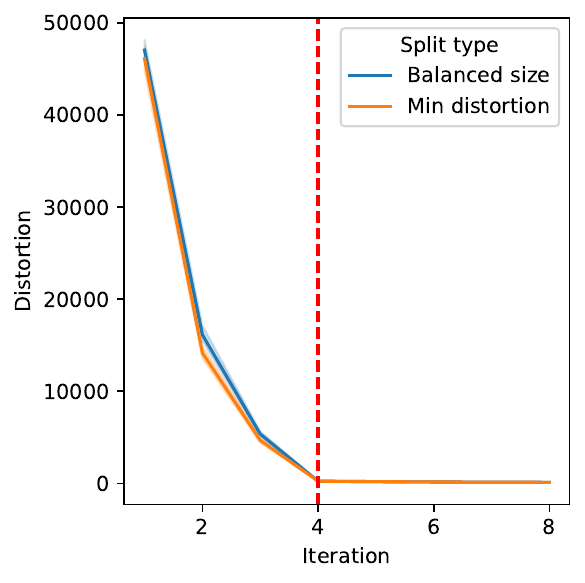}
\includegraphics[width=.2\linewidth]{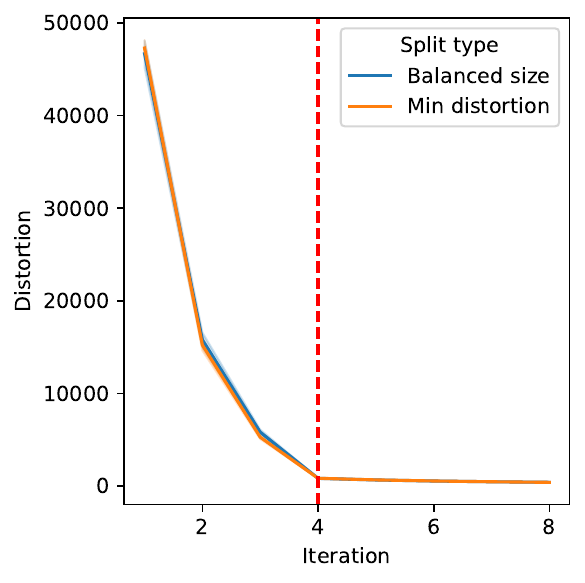}
\includegraphics[width=.2\linewidth]{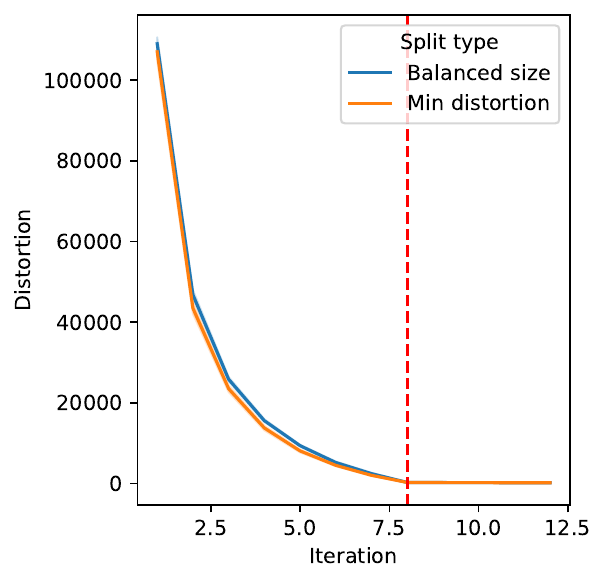}
\includegraphics[width=.2\linewidth]{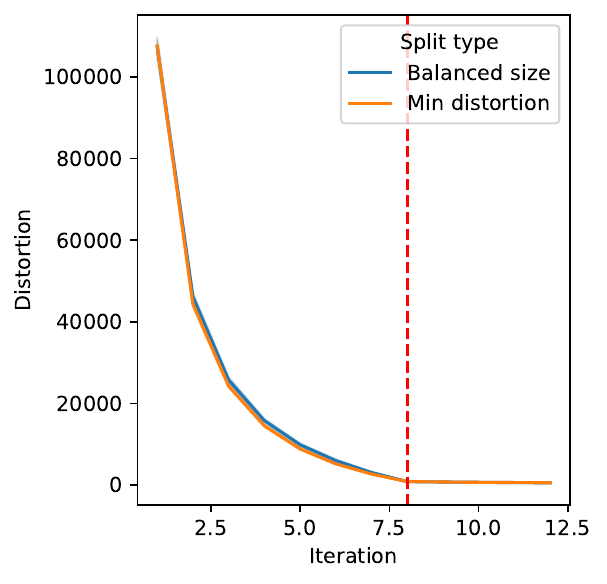}\\
\includegraphics[width=.2\linewidth]{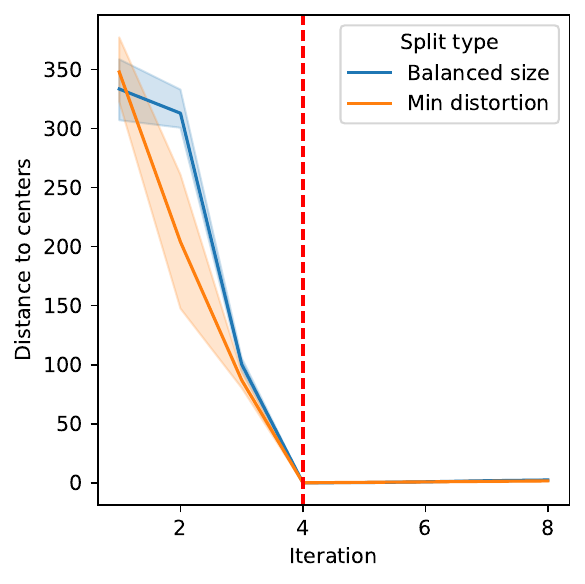}
\includegraphics[width=.2\linewidth]{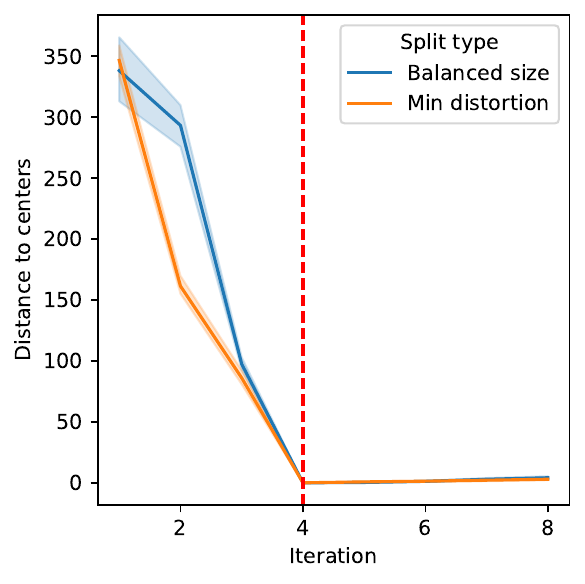}
\includegraphics[width=.2\linewidth]{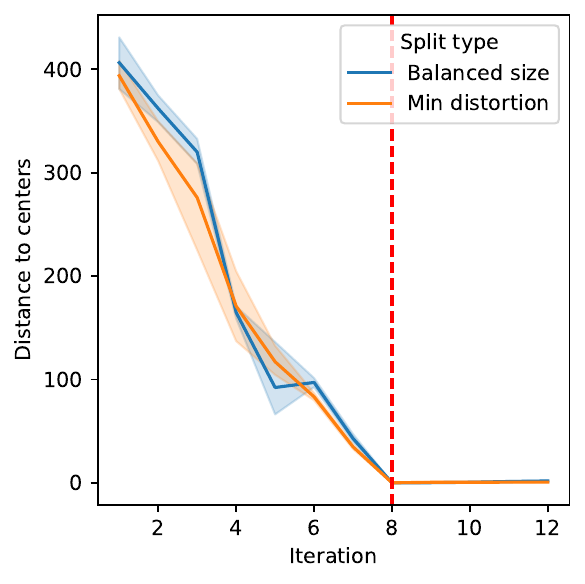}
\includegraphics[width=.2\linewidth]{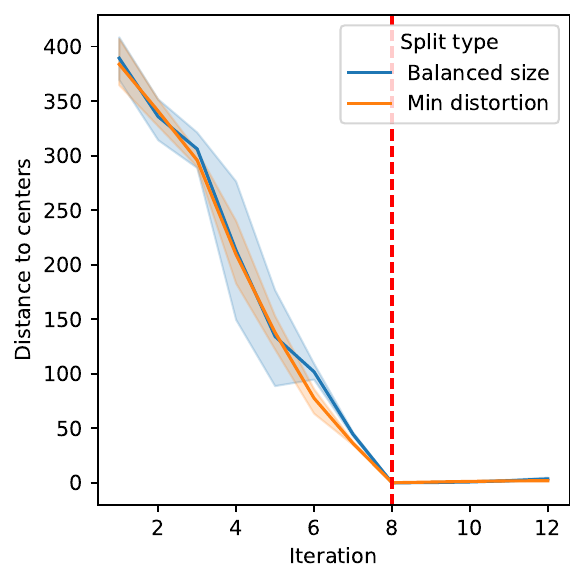}
\caption{Components are Mallows models of $n=50$. $k=4$ in (a) and (b), $k=8$ in (c) and (d). Also, (a) and (c) are more concentrtaed than (b) and (d) }
\end{figure}

\begin{figure}[h]
\centering
\includegraphics[width=.2\linewidth]{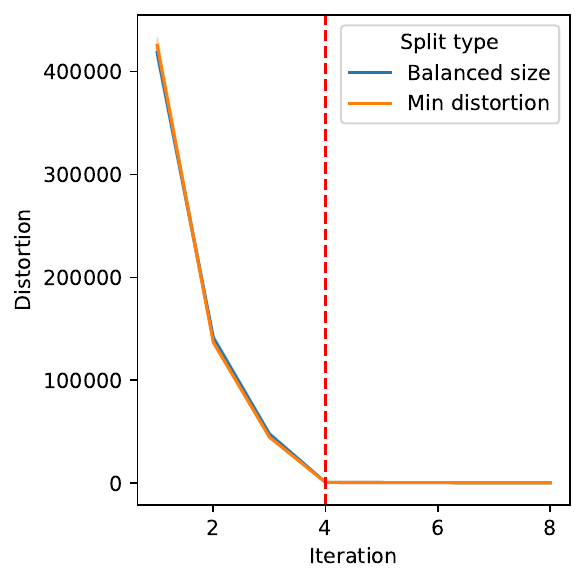}
\includegraphics[width=.2\linewidth]{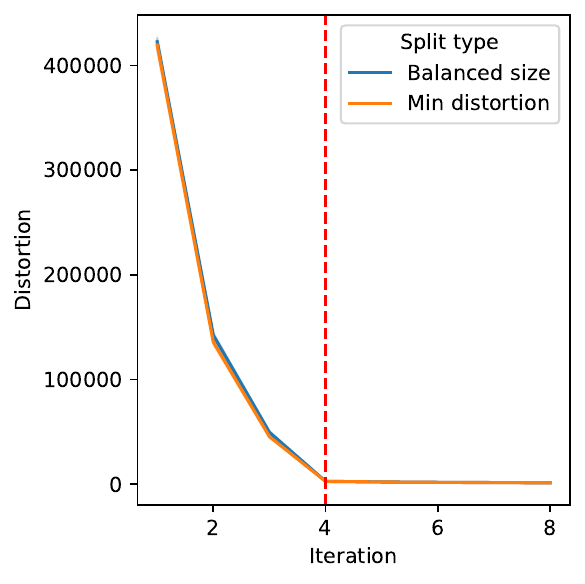}
\includegraphics[width=.2\linewidth]{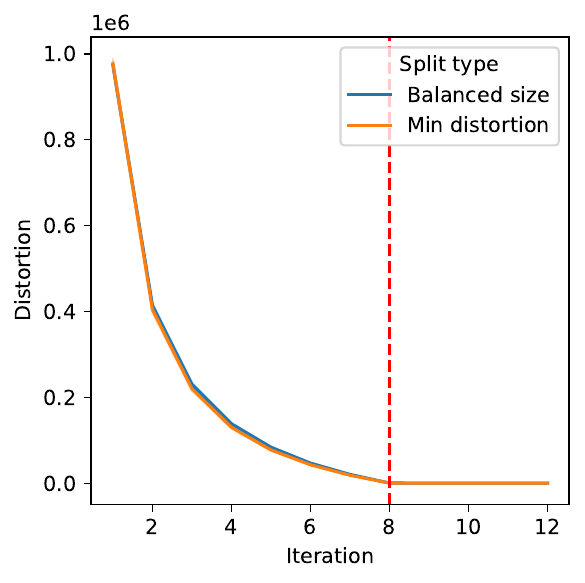}
\includegraphics[width=.2\linewidth]{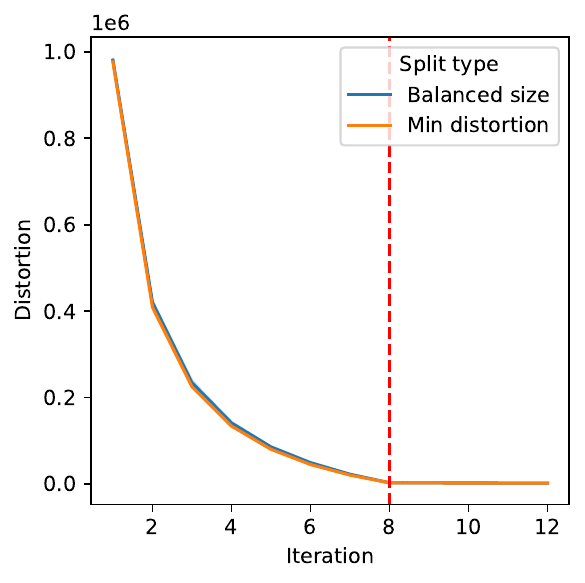}\\
\includegraphics[width=.2\linewidth]{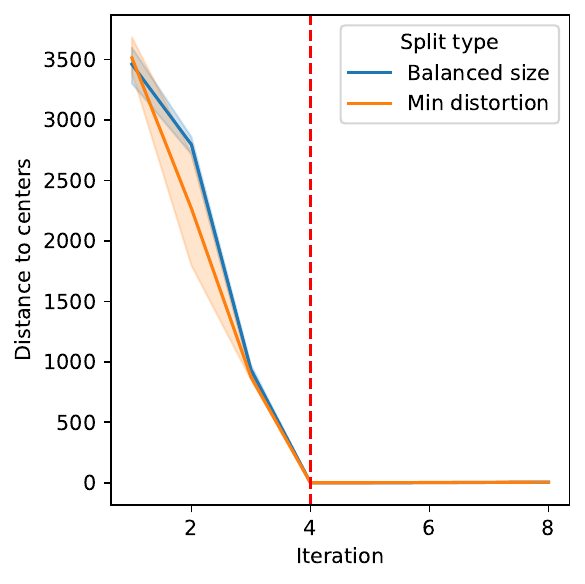}
\includegraphics[width=.2\linewidth]{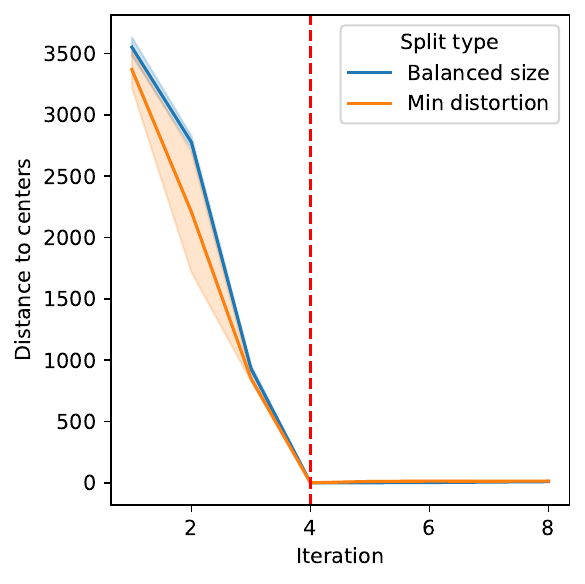}
\includegraphics[width=.2\linewidth]{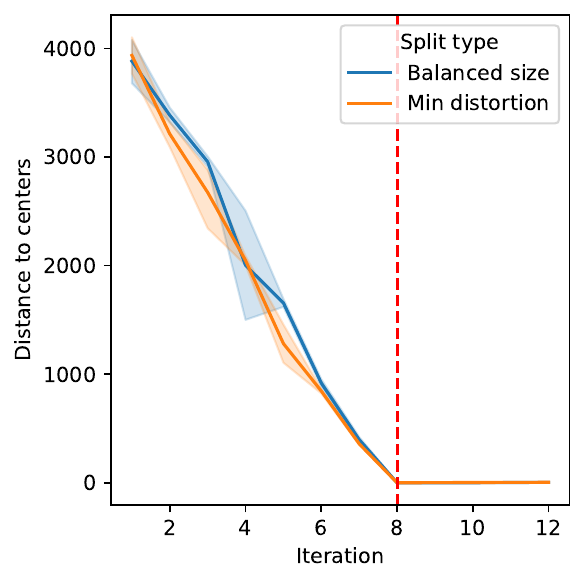}
\includegraphics[width=.2\linewidth]{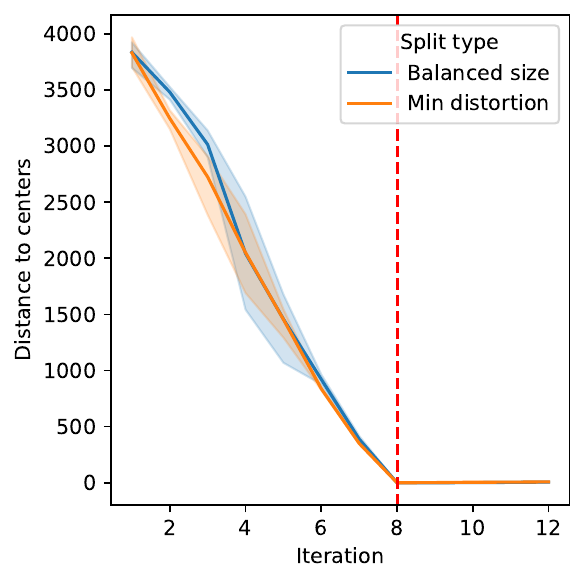}

\caption{Components are Mallows models of $n=50$. $k=4$ in (a) and (b), $k=8$ in (c) and (d). Also, (a) and (c) are more concentrated than (b) and (d) }
\end{figure}

\begin{figure}[h]
\centering
\includegraphics[width=.2\linewidth]{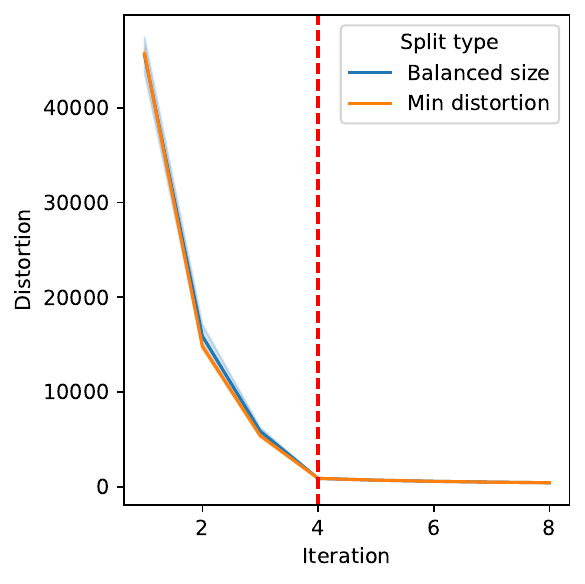}
\includegraphics[width=.2\linewidth]{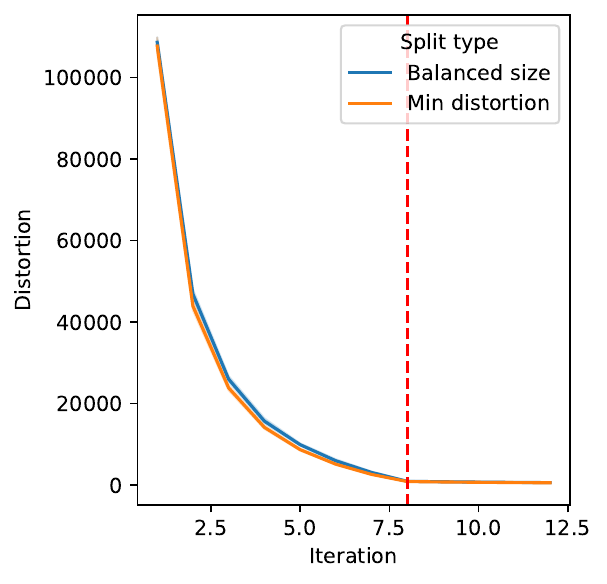}
\includegraphics[width=.2\linewidth]{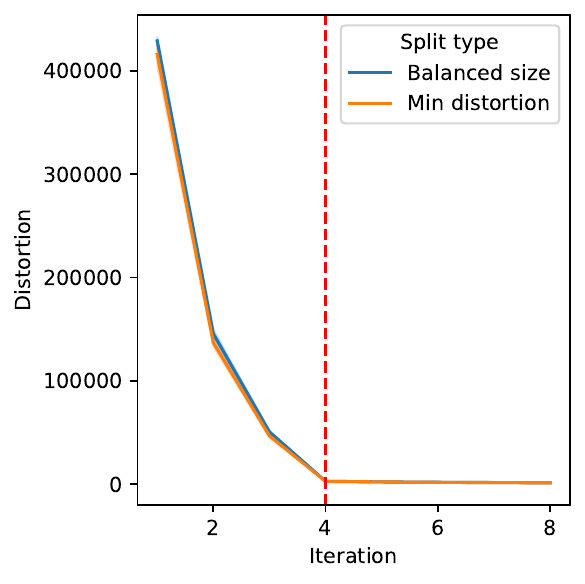}
\includegraphics[width=.2\linewidth]{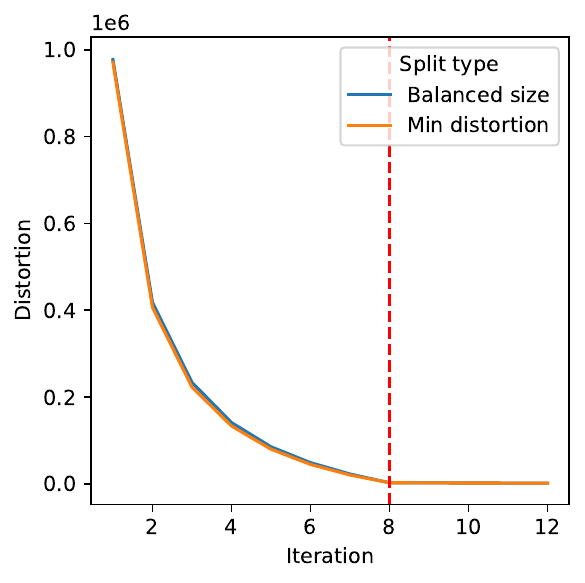}\\
\includegraphics[width=.2\linewidth]{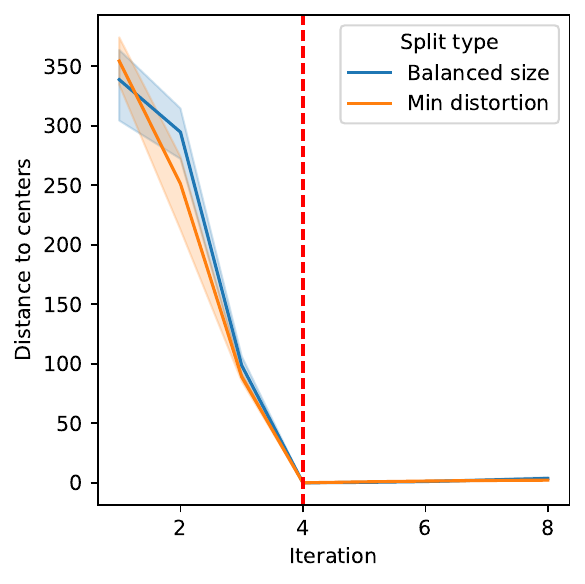}
\includegraphics[width=.2\linewidth]{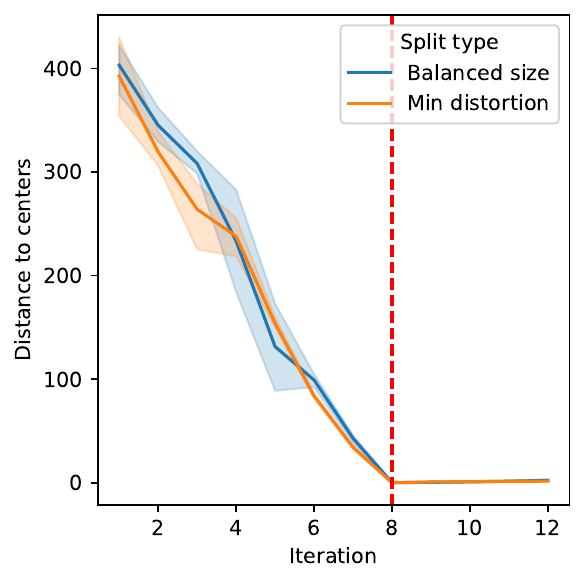}
\includegraphics[width=.2\linewidth]{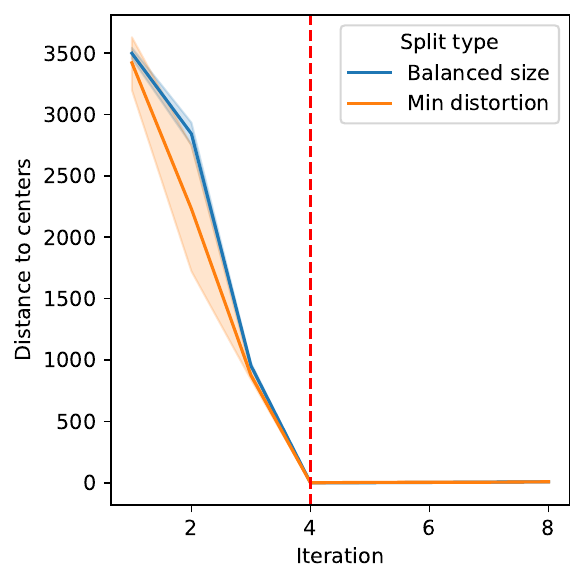}
\includegraphics[width=.2\linewidth]{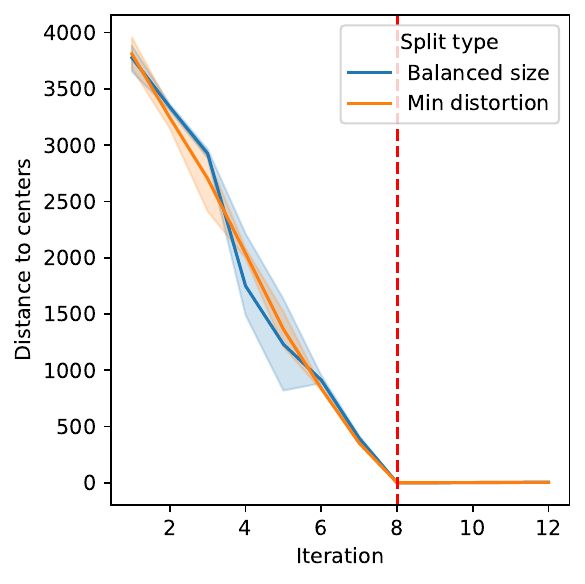}

\end{figure}

\begin{table}[h]
\centering
\begin{tabular}{c c c c}
\hline
$n$ & $k$ & Balanced size  & Min distortion \\
\hline
10  & 4 & 0.0826   & 0.0923 \\
50  & 4 & 0.5197   & 3.5002 \\
100 & 4 & 1.7686   & 48.4428 \\
150 & 4  & 4.372  & 168.955 \\
\hline
10  & 8 & 0.2424   & 0.3229 \\
50  & 8 & 1.4242   & 13.9436 \\
100 & 8 & 4.6943   & 161.861 \\
150 & 8 &    15.022  & 1094.133 \\
\hline
\end{tabular}
\end{table}

\subsubsection{Recovering the partitions in very noisy settings}

In the main paper, we present the primary application of the COAST algorithm and study its ability to identify the components of a mixture through the evolution of the distortion across iterations, as well as through the distance between the recovered centers and the true centers of the underlying distribution, in the case where the components are equally spread. 

Here, we extend this analysis to a noisier setting with heterogeneous scales across components. We consider a sample drawn from a Mallows mixture (MM) distributions with $n=20$ and $K=4$ components and display the co-membership structure obtained after $K$ iterations of the COAST algorithm. In Figure~\ref{fig:co-members}, the first row shows a heatmap of the pairwise distance matrix $D$, where $D_{ij}$ denotes the Kendall's-$\tau$ distance between rankings $\sigma_i$ and $\sigma_j$. The second row displays the cluster co-membership matrix $\tilde D$, where $\tilde D_{ij}$ is shown in green if and only if $\sigma_i$ and $\sigma_j$ belong to the same partition $\cal C$ after $K$ iterations.

As can be observed, even under highly noisy conditions, the partitions corresponding to the original mixture components are almost perfectly recovered.

\begin{figure}[h]
\centering
\includegraphics[width=.2\linewidth]{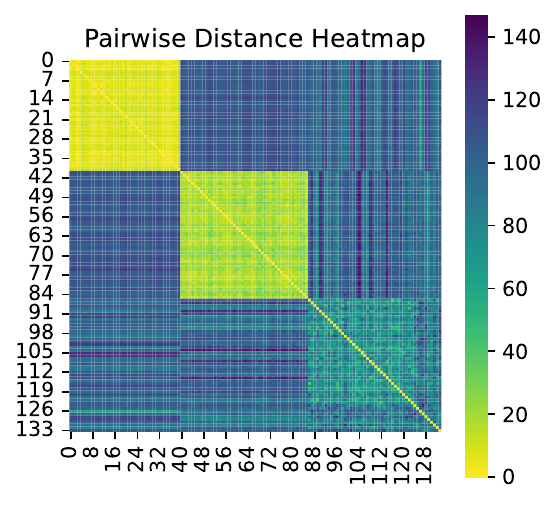}
\includegraphics[width=.2\linewidth]{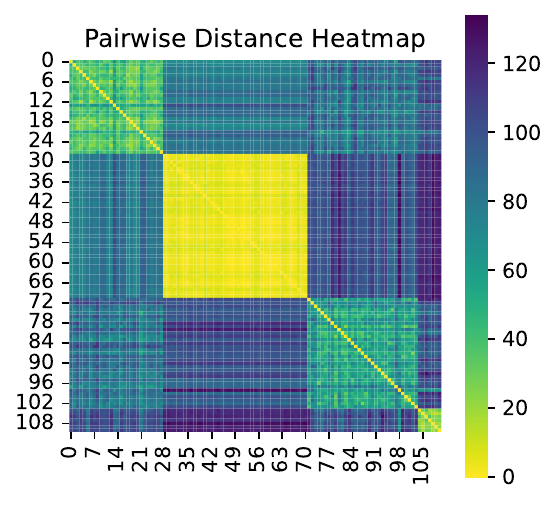}
\includegraphics[width=.2\linewidth]{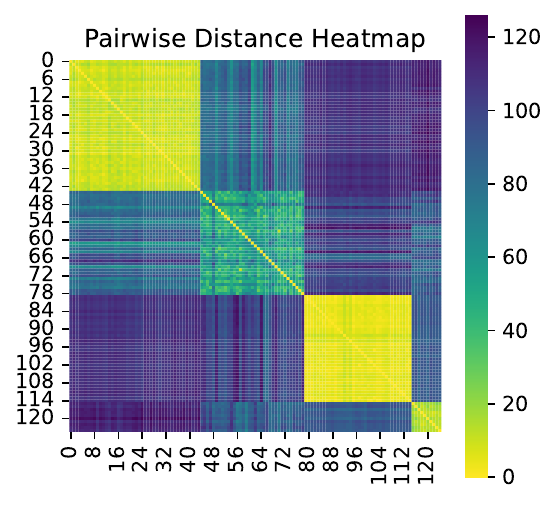}\\
\includegraphics[width=.2\linewidth]{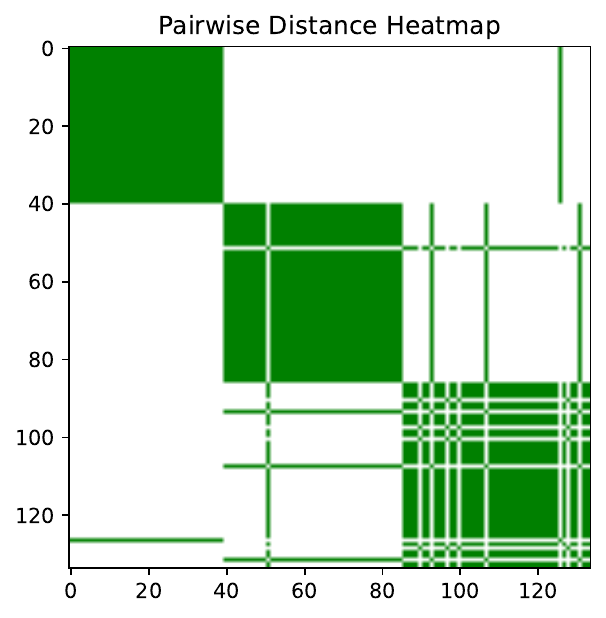}
\includegraphics[width=.2\linewidth]{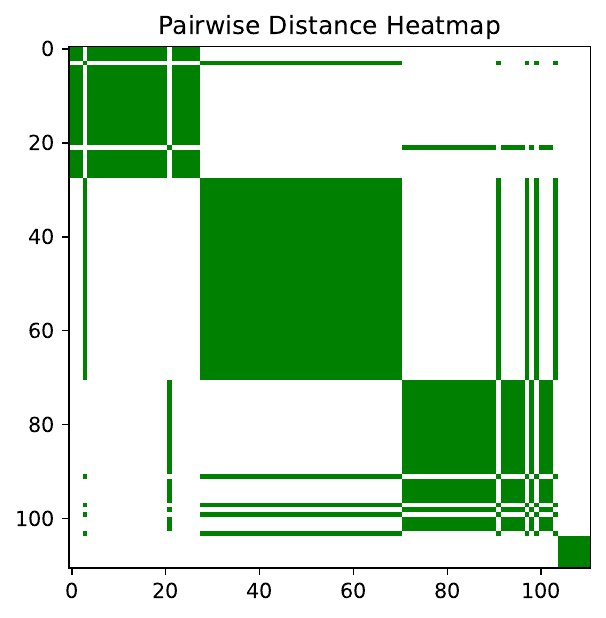}
\includegraphics[width=.2\linewidth]{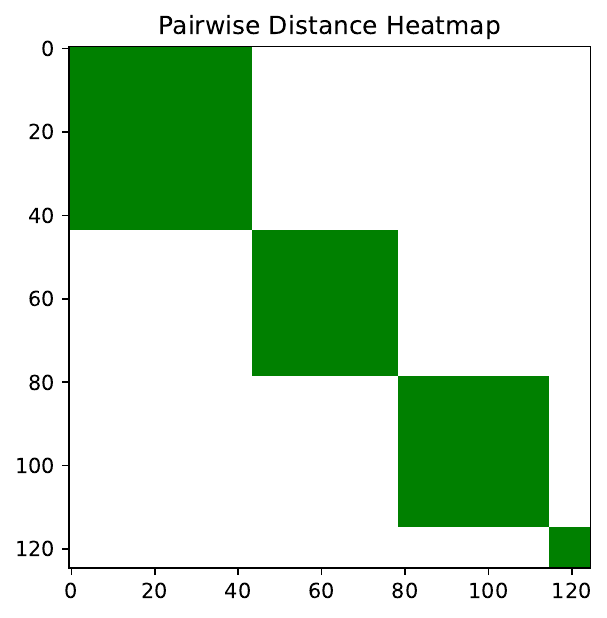}
\caption{Heatmap of the distances between pairs of rankings of noisy ranking models (Mallows model of random dispersion and $K=4$ components, top row) and co-membership matrix after 4 iterations of the COST algorithm (bottom row).}
\label{fig:co-members}
\end{figure}

\subsection{Smoothing: depths are a good approximation for ranking distributions and CRD}
In Equation~\eqref{eq:smooth}, we propose a smoothing technique for the ranking distribution within each component $ \widetilde{P}_{\mathcal{C} }(\sigma)$ based on the local depth of the rankings associated with each partition $D_{P ( {\mathcal{C} } ) } (\sigma)$. Since the expression of this local depth is trivial given that the local distribution can be approximated,  $D_{P ( {\mathcal{C} } ) } (\sigma)=  \sum_{i<j} p_{\sigma^{-1}(i),\,\sigma^{-1}(j)}(\mathcal{C})$, the main challenge in this setting lies in the computation of the normalizing constant $Z_{\mathcal C}$. 

In order to give an analytical expression for $Z_{\mathcal C}$, we mimic the derivation of the expression for the variability of a distribution $P$  written in the main paper as $V'_P = \sum_{i<j} p_{i,j}\bigl(1 - p_{i,j}\bigr)$: We recast this constant in terms of the variability induced by the uniform distribution on the partition $\mathcal C$. In particular, we write
\[
Z_{\mathcal C} = \frac{1}{2}\,\mathbb{E}\!\left[d(\Sigma,\Sigma') \big|\, \Sigma' \sim U_{\mathcal C} , \Sigma \sim P\right].
\]
This can be equivalently expressed as
\[
Z_{\mathcal C} = \sum_{i<j} p_{i,j}\bigl(1 - p'_{i,j}\bigr),
\]
where
\[
p'_{i,j}
=
\mathbb{P}\bigl\{ \Sigma(i) < \Sigma(j) \,\big|\, \Sigma \sim U_{\mathcal C} \bigr\}.
\]

The main difficulty in evaluating this expression lies in the computation of the probabilities $p'_{i,j}$. Closed-form expressions are trivial in several important cases, including $\mathcal C = S_n$, partially observed rankings \cite{himmi-etal-2024-towards}, and the degenerate case $\mathcal C = \sigma$.

In the COAST algorithm with the alternative splitting criterion, the leaves of the tree consist on the partitions $\cal C$ that consist on a sequence of splits on pairs $(i,j) \in {\cal J}$ where each item $i$ is only used once in the each $\cal C$. Let $U_{\mathcal C}$ denote the uniform distribution on the partition $\mathcal C$. Using basic enumerative combinatorics observations, it can be factorized as
\[
U_{\mathcal C} = \prod_{(i,j) \in \mathcal I(\mathcal C)} U_{i,j},
\]
where $\mathcal I(\mathcal C)$ denotes the set of unconstrained pairs induced by the partition $\mathcal C$. For each $(i,j) \in \mathcal I(\mathcal C)$, the matrix $U_{i,j}$ is defined entrywise by
\[
(U_{i,j})_{a,b}
=
\begin{cases}
1, & \text{if } a=i \text{ and } b=j, \\[0.2em]
\frac{1}{2}, & \text{if } a,b \neq i,j, \\[0.2em]
\frac{1}{3}, & \text{if } a=i \text{ and } b \neq j, \\[0.2em]
\frac{2}{3}, & \text{if } a \neq i \text{ and } b=j .
\end{cases}
\]

assuming $a<b$ and setting $(U_{i,j})_{b,a} = 1-(U_{i,j})_{a,b}$.

\subsection{Local depth and anomaly detection}
This section presents additional experiments in the same setting for anomaly detection from the main paper. In particular, we consider 20 test samples from Mallows mixtures (MM) with $n=30$ and $K=4$ (a) and $K=8$ (b), as well as Plackett-Luce mixtures (PL) with exponentially decreasing weights, $n=30$ and $K=4$ (c) and $K=8$ (d).

\begin{figure}[h]
\centering
\includegraphics[width=.2\linewidth]{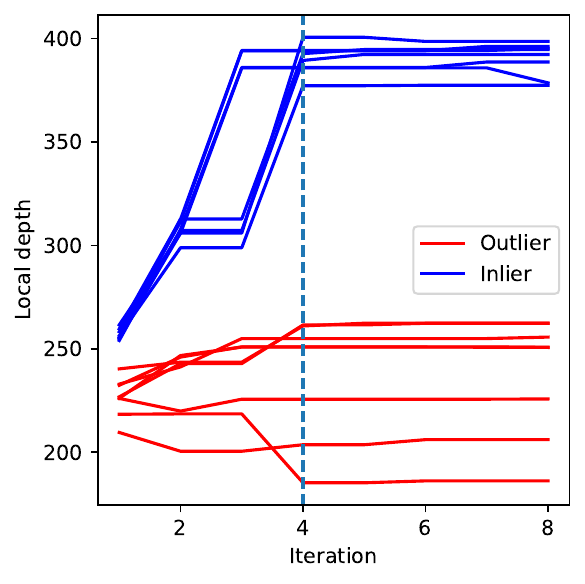}
\includegraphics[width=.2\linewidth]{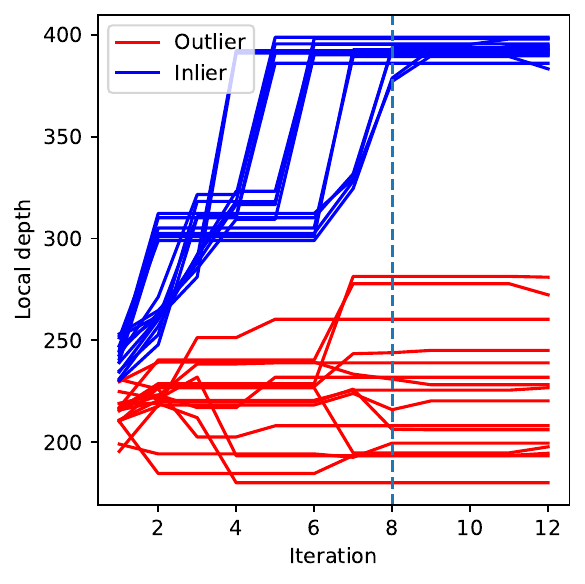}
\includegraphics[width=.2\linewidth]{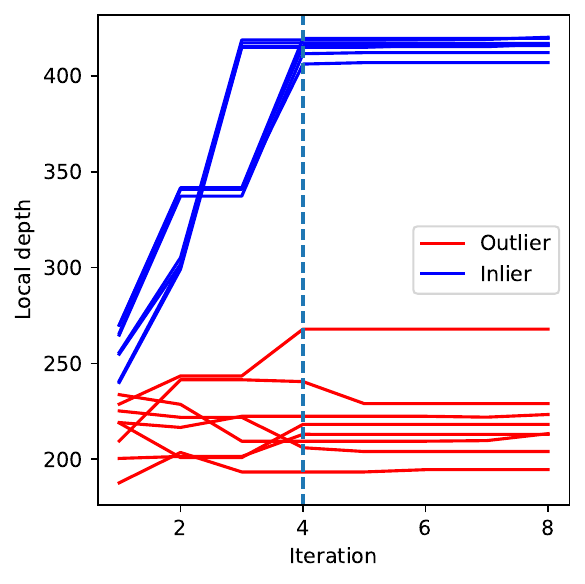}
\includegraphics[width=.2\linewidth]{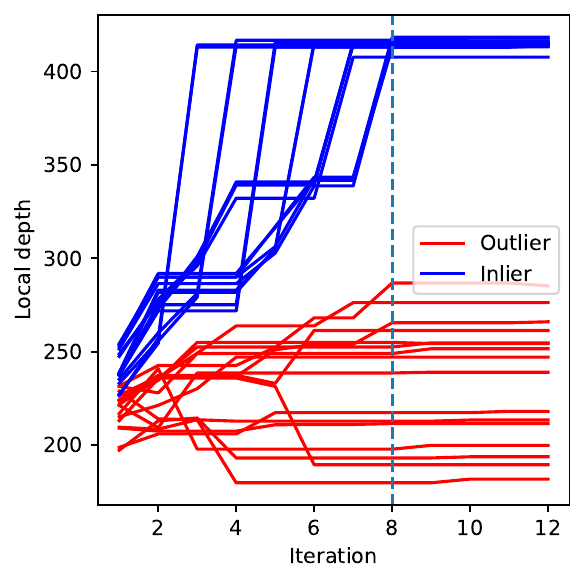}
\end{figure}

\subsection{Local medians}

In this section, we analyze the behavior of local medians when conditioning on the partition structure defined as a tree in the COAST algorithm.

We start by considering a single component that follows an SST distribution and condition it on pairs $(i,j) \in \mathcal I(\mathcal C)$ for some partition $\mathcal C$. Let $P^{(l)}$ be the distribution after conditioning $l$ times.  The next Lemma shows that there exists a sequence of partitions that make $P^{(l)}$ be SST for all $l$ provided that $P^{(0)}$ is SST. In particular, we prove it for one conditioning step, the rest follows by recursion. It follows that this sequence defines a sequence of depth-trimmed regions whose median coincides with the true median.

\begin{lemma}
    Let $P$ be SST, let $(a,b) \in \arg\max P_{i,j}$ and $\mathcal C = \{ \sigma : \sigma(a) < \sigma(b) \}$. Then $P({\cal C})$ is SST. 
\end{lemma}

\begin{proof}
Without loss of generality, assume that $P$ is an SST distribution centered at the identity permutation.  
The SST condition means that $P_{i,k} \geq  \max (P_{i,j}, P_{j,k})$ for all $i<j, j<k$, i.e., $P$ is increasing along each row from left to right and along each column from bottom to top.

Again without loss of generality, we consider a single conditioning step.  
Let $P$ be an SST distribution and let $P(\mathcal C)$ denote the distribution obtained after conditioning on the event
$\mathcal C = \{ \sigma : \sigma(a) < \sigma(b) \}.$

We first observe that if $(a,b)$ is not an index achieving $\arg\max_{i,j} P_{i,j}$, then in general, for all $i>a$ and $j<b$, one has
\(
P(\mathcal C)_{i,j} \neq 1,
\)
and therefore the SST condition does not hold (unless the implication $\sigma(a) < \sigma(b) \Rightarrow \sigma(i) < \sigma(j)$ holds).

Next, let $(a,b)$ be such that
\[
P_{a,b} > P_{i,j} > P_{i,j-1},
\]
where the second inequality follows from the SST property. We use the notation $P(\{ a \prec b \} ) = \mathbb{P} (\Sigma | \Sigma(a)<\Sigma(b))$. 
Then, for all $i,j$, we have
\begin{align}
\begin{split}
    P(\{ a \prec b \} \cap \{ i \prec j \})
&\geq P(\{ a \prec b \} \cap \{ i \prec j-1 \}), \\
P(\{ i \prec j \mid a \prec b \}) \, P_{a,b}
&\geq P(\{ i \prec j-1 \mid a \prec b \}) \, P_{a,b}, \\
P(\{ i \prec j \mid a \prec b \})
&\geq P(\{ i \prec j-1 \mid a \prec b \}), \\
P(\mathcal C)_{i,j}
&\geq P(\mathcal C)_{i,j-1}.
\end{split}
\label{eq:sst_c}
\end{align}

The first inequality follows from this observation: for all permutation $\sigma\in \{ a \prec b \} \cap \{ i \prec j \}$ we can invert positions $j,j-1$ and the result $\sigma\tau_{j,j-1}\in \{ a \prec b \} \cap \{ i \prec j-1 \}$. We finish noting that $\mathbb{P}(\sigma) \geq \mathbb{P}(\sigma\tau_{j,j-1})$ for all SST distributions. 

Equation~\eqref{eq:sst_c} shows that $P(\mathcal C)$ is increasing along rows. A similar argument applies to columns, and therefore $P(\mathcal C)$ satisfies the SST property.
\end{proof}

\begin{remark}
     The conditional distribution $P_{\mathcal C}^{(l)}$ is not SST in general. However, even when the pairs $(i,j)$ are chosen randomly, in practice, provided that the pairs are consistent with the median, i.e., $P_{i,j}>.5$ the recovered median is a good approximation of the median, specially when the distribution is concentrated.
\end{remark}

To illustrate these points, we include a randomized rank aggregation algorithm as an example. This algorithm, presented in \cite{irurozki:hal-03972357}, is essentially a Markov chain that starts from a random ranking and moves at each step to the neighboring ranking (at distance 1) with the largest depth. 

Let us consider a sample $\{ \sigma_1, \ldots, \sigma_m\}$ drawn from a SST distribution with center $\sigma_0$ and empirical measure $\widehat P$. In the plots in Figure~\ref{fig:local_med_1_compo}, each point represents a ranking $\sigma$ in the sample, placed at coordinates $(d(\sigma, \sigma_0), D_{\hat P}(\sigma))$. The goal of the algorithm is therefore to reach the top-left-most point, corresponding to $\sigma_0$. 

The lines indicate the paths of the MCMC chain over 5 runs, which, by construction, always move "up" in depth. While the $x$-axis is not known to the algorithm in practice, as shown in \cite{irurozki:hal-03972357}, increasing depth $D_{\hat P}(\sigma)$ corresponds to decreasing in $d(\sigma, \sigma_0)$ (moving left) for all SST distributions. 

Figure~\ref{fig:local_med_1_compo}(a) illustrates the behavior of this Markovian algorithm on a distribution with a single SST component. 
Let $P^{(l)}$ be the distribution after conditioning $l$ times.  
In panel (b) we consider the partition $\cal C$ that consists on recursively conditioning 3 times, always on the pairs $(i,j) = \arg\max P_{i,j}^{(l)}$ for $l \leq 3$.

Panels (c) and (d) repeat the process 6 and 9 times respectively. Note that, or $n=6$, in at most 14 conditioning steps, the median becomes the only element in the partition $\mathcal C$. 
By construction, $D_{ P_{\cal C}}$ is SST for every $\cal C$. Therefore, and for the same proof as in \cite{irurozki:hal-03972357}, the Markovian algorithm is guaranteed to find the median $\sigma_0$. The plots confirm this fact.

\begin{figure}[h]
\centering
\includegraphics[width=.2\linewidth]{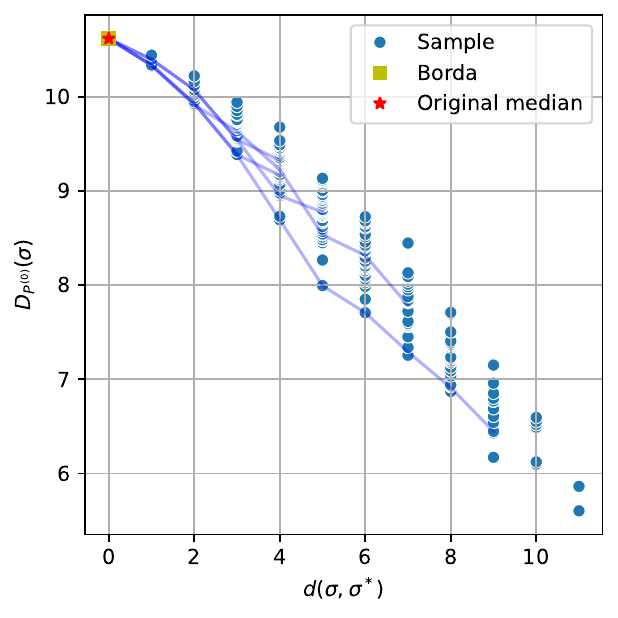}
\includegraphics[width=.2\linewidth]{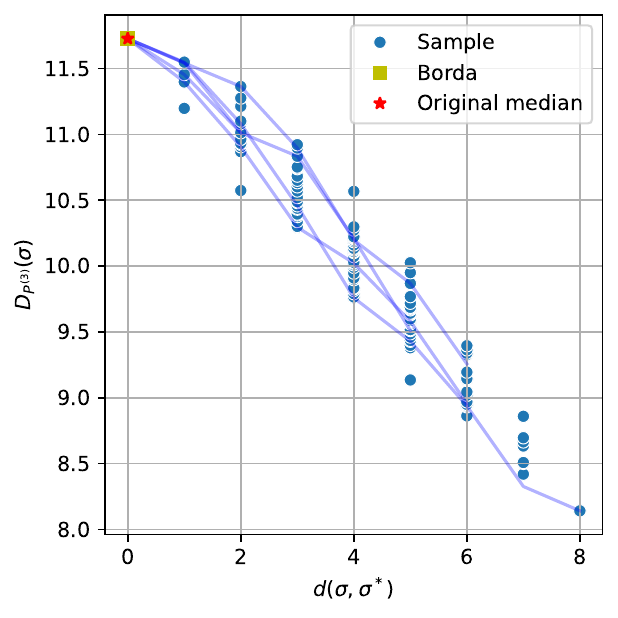}
\includegraphics[width=.2\linewidth]{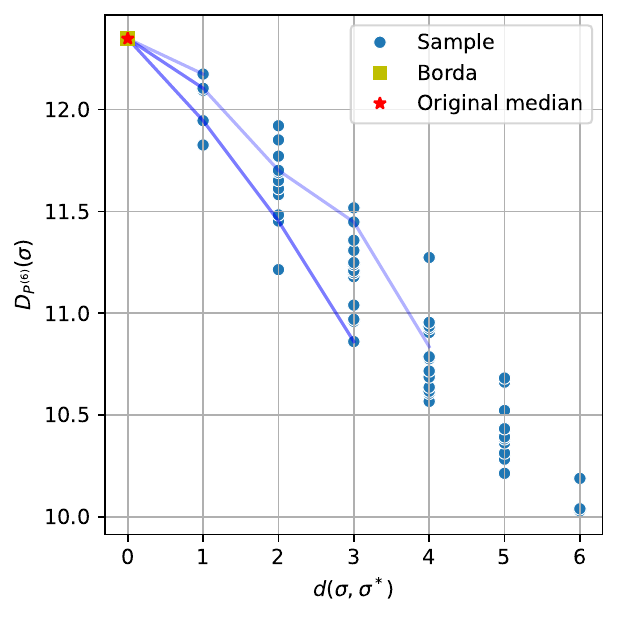}
\includegraphics[width=.2\linewidth]{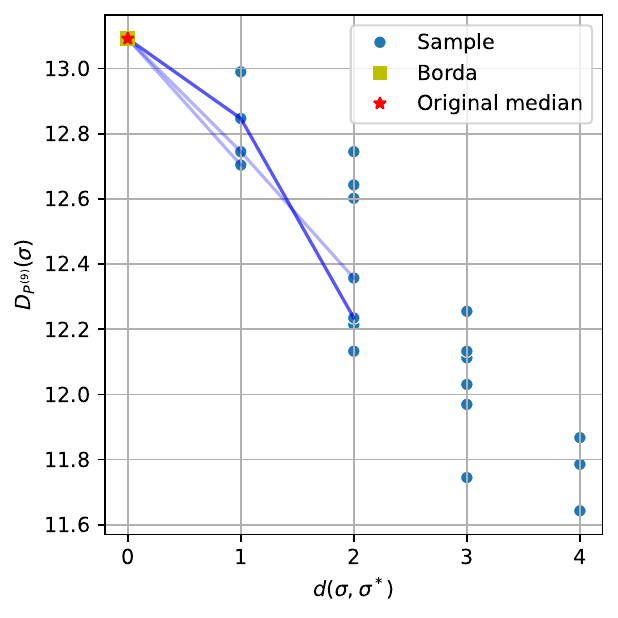}
\caption{Markovian rank-aggregation algorithm for the case of a simple SST component. The sequence of pairs $(i,j) \in \mathcal I(\mathcal C)$ is chosen such that $(i,j) = \arg\max P_{i,j}$, so that $P({\mathcal C})$ is SST at every conditioning step of the algorithm. We include the results of conditioning 0, 3, 6 and 9 times for $n=6$.}
\label{fig:local_med_1_compo}
\end{figure}

As pointed out above, when conditioning $l$ times on arbitrary pairs, the resulting sequence $P^{(l)}$ is not SST in general. Nevertheless, the algorithm still performs well in practice. As shown in \cite{irurozki:hal-03972357}, the procedure is robust: even though not all depth-increasing paths converge to the local median, the majority of them do, allowing the median to be reliably recovered. See the illustration of the CRD of $K=4$ on Figure~\ref{fig:local_med_k_comp}.

\begin{figure}[h]
\centering
\includegraphics[width=.2\linewidth]{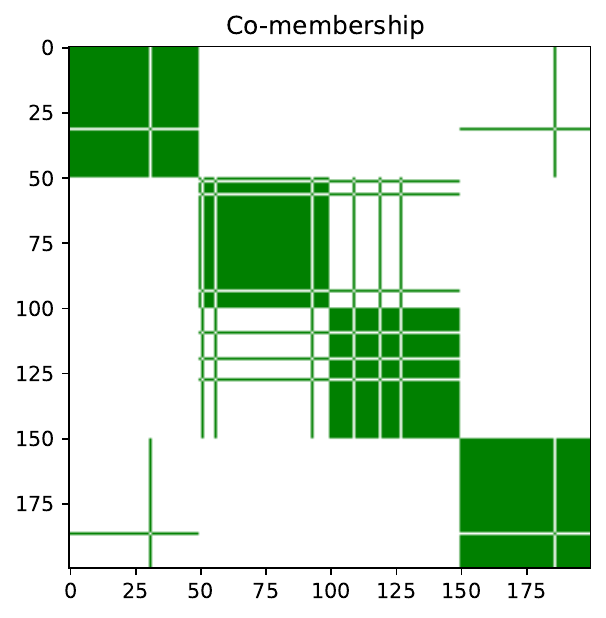}
\includegraphics[width=.2\linewidth]{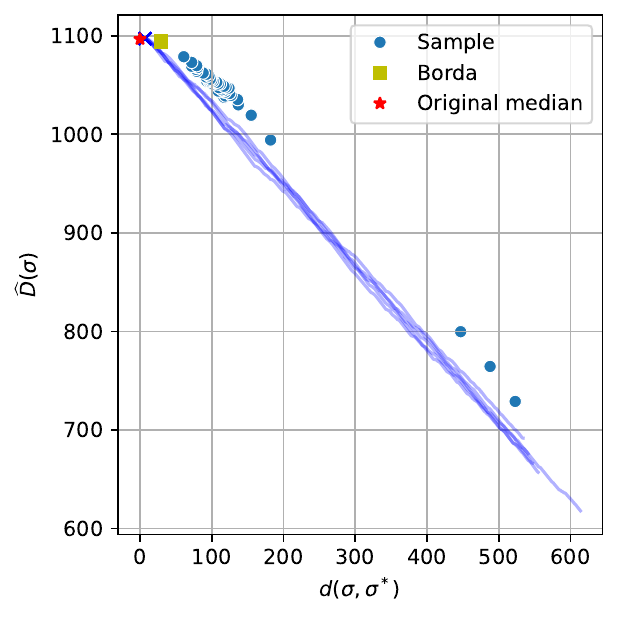}
\includegraphics[width=.2\linewidth]{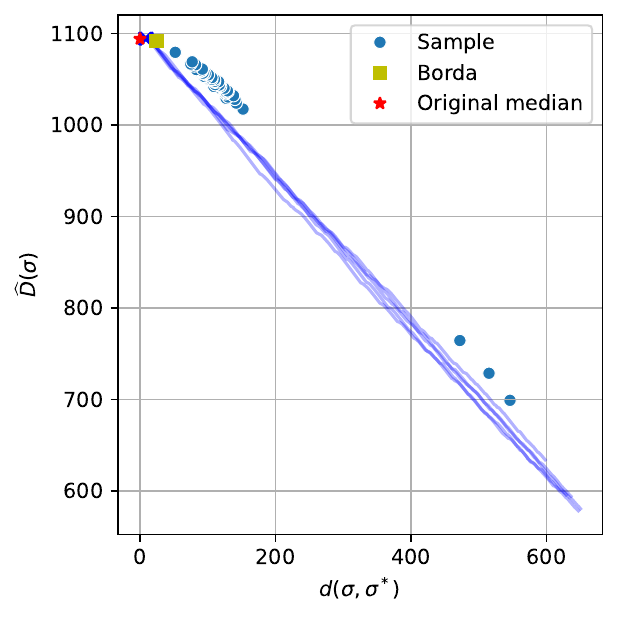}
\includegraphics[width=.2\linewidth]{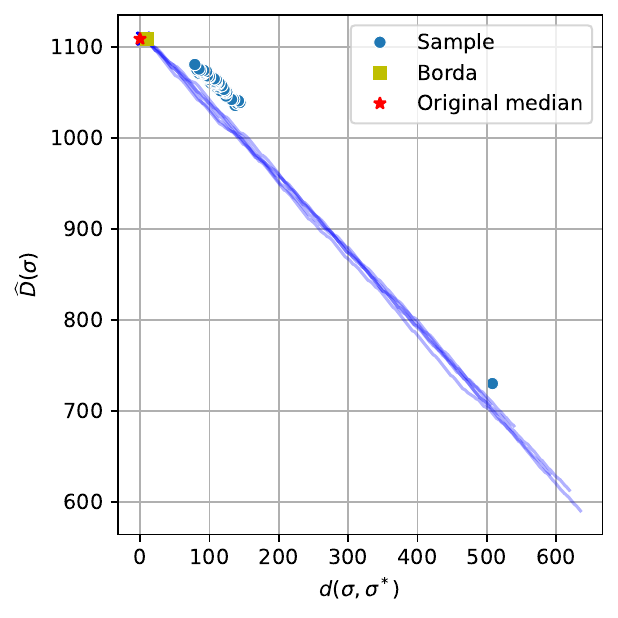}
\caption{Markovian rank-aggregation algorithm for the case of a simple SST component. The sequence of pairs $(i,j) \in \mathcal I(\mathcal C)$ is chosen such that $(i,j) = \arg\max P_{i,j}$, so that $P({\mathcal C})$ is SST at every conditioning step of the algorithm. We include the results of conditioning 0, 3, 6 and 9 times for $n=6$.}
\label{fig:local_med_k_comp}
\end{figure}

\subsection{Real-world data: Sushi dataset}
We now evaluate the interpretability and performance of COAST on real-world data, specifically the Sushi dataset \cite{sushi}. Figure~\ref{fig:sushi_dist} shows the evolution of the distortion and likelihood of the COAST algorithm over 30 iterations. After approximately 20 iterations, the distortion plateaus, indicating that most of the variability in the data has already been captured. 

This experiment compares the two splitting criteria: the exact criterion (referred to as "Min distortion") and the approximate criterion ("Balanced size"). The approximate criterion was run 10 times. Both criteria yield very similar distortion values, demonstrating that the approximation is effective. 

Furthermore, the approximate criterion enables additional analysis through the normalization constant defined in Equation~\eqref{eq:smooth}. In particular, the smoothing function can be interpreted probabilistically, allowing the use of a likelihood-based criterion, as shown in Figure~\ref{fig:sushi_dist}(b) from which the same conclusion can be confirmed.

\begin{figure}[h]
\centering
\includegraphics[width=.2\linewidth]{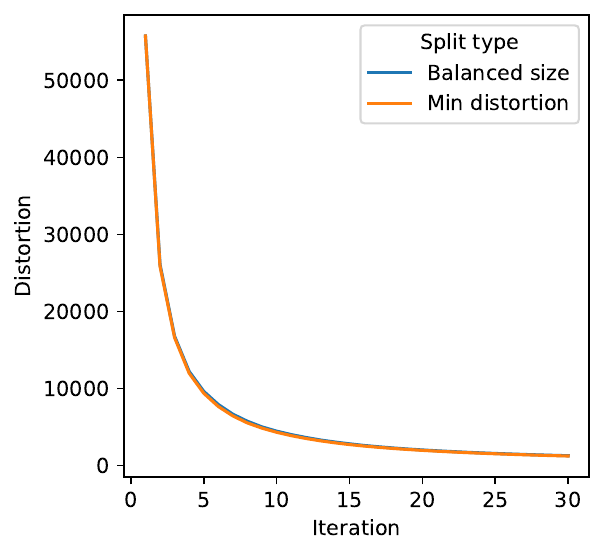}
\includegraphics[width=.2\linewidth]{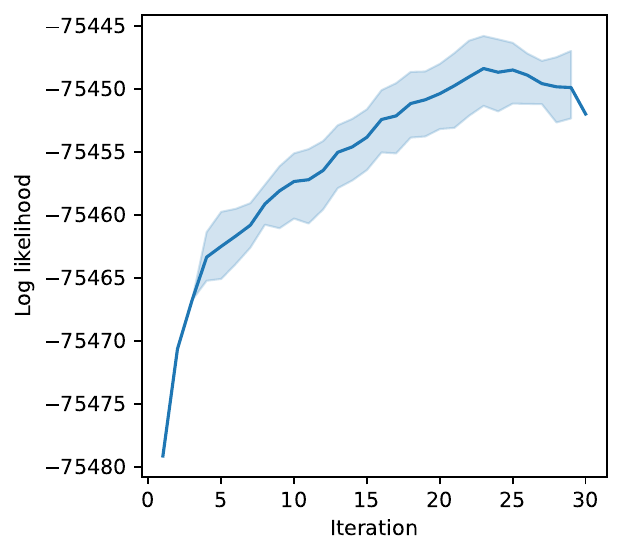}
\caption{Evolution of the distortion and likelihood of the COAST algorithm in 30 iterations}
\label{fig:sushi_dist}
\end{figure}

Moreover, the partition structure can be visualized to aid interpretation, Figure~\ref{fig:tree_sushi}. The results of the first five iterations of the algorithm, using the deterministic criterion "Min distortion," are depicted as a tree. The first partition separates the population based on sushi types 4 and 3. This indicates that individuals who rank sushi 3 before sushi 4 form a distinct subgroup, and consequently, their rankings of the remaining sushi types differ from those who prefer sushi 4 to sushi 3. For each leave ${\cal C}$ in the tree from left to right, $P({\cal C})$ is 0.13, 0.14, 0.12, 0.14, 0.12, 0.12, 0.1, 0.13, and their in-component variability $V'$ is  8.88, 
8.81, 
8.92, 
8.72, 
9.0, 
8.87, 
8.81, 
8.69.

The key takeaway is that this tool not only performs clustering but also serves as a valuable instrument for gaining insight into the underlying structure of the data and understanding the relationships between different elements in the ranking problem.

\begin{figure}[h]
\centering
\includegraphics[width=.9\linewidth]{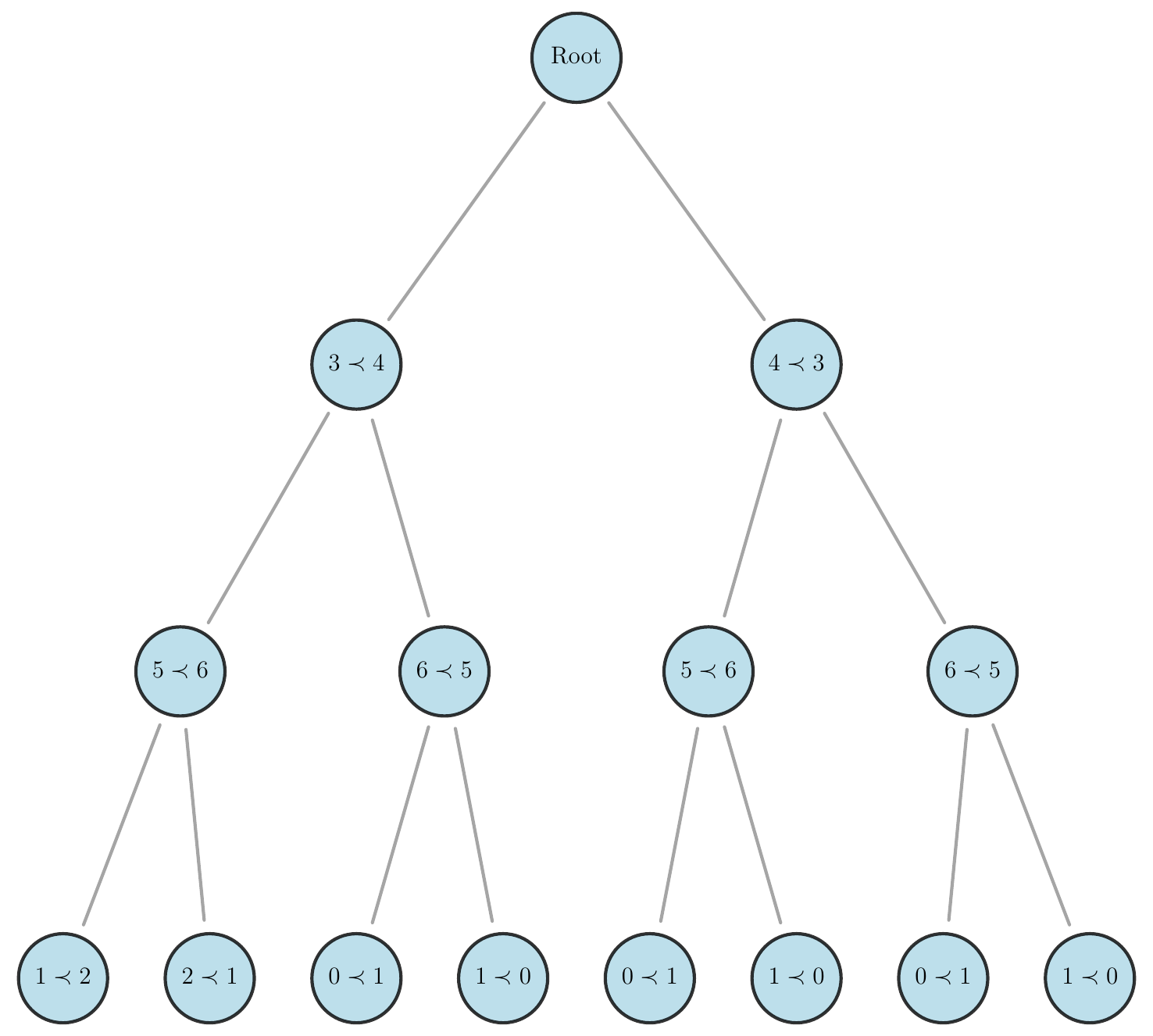}
\caption{Tree structure recovered after 8 iterations of the COAST algorithm}
\label{fig:tree_sushi}
\end{figure}

\subsection{Homogeneity testing}
Depths, both local and global, can also be employed to provide formal inference. We illustrate this through a nonparametric test of homogeneity between two Mallows model mixtures with $n=6$. Four mixtures are considered, each with varying numbers of components $K$ and scale parameters $\phi$ (see Table~\ref{tab:dist_params}). We perform the Wilcoxon rank-sum test comparing distribution 1 with each of the remaining distributions (i.e., 1 vs.\ 2, 1 vs.\ 3, and 1 vs.\ 4).

\begin{table}[h!]
\centering
\begin{tabular}{c c c }
\hline
Distribution & $\phi$ & $K$ \\
\hline
1 & 0.2 & 8 \\
2 & 0.2 & 8 \\
3 & 0.05 & 8 \\
4 & 0.2 & 4 \\
\hline
\end{tabular}
\caption{Parameters of the four distributions used in the experiment.}
\label{tab:dist_params}
\end{table}

The Wilcoxon rank-sum test is performed using a reference sample of size 100 drawn from all distributions. The results, averaged over 200 repetitions, are shown in Figure~\ref{fig:homo_test}. The p-values effectively detect differences between distributions when they exist, providing a formal inference that complements the visualization offered by ranking DD-plots. In particular, during the first iteration of the algorithm, the mixture $P$ is effectively very uniform, and the resulting samples are consistently evaluated as different. However, when the distributions differ in $K$ or the scale parameter $\phi$, the test quickly identifies these differences, producing p-values close to zero by the second or third iteration.  
Remarkably, this procedure does not rely on any parametric assumptions about the underlying ranking models.

The key takeaway is that ranking depths in a CRD can be used to construct goodness-of-fit statistics, allowing one to assess how well a particular ranking model explains a given dataset.

\begin{figure}[h]
\centering
\includegraphics[width=.5\linewidth]{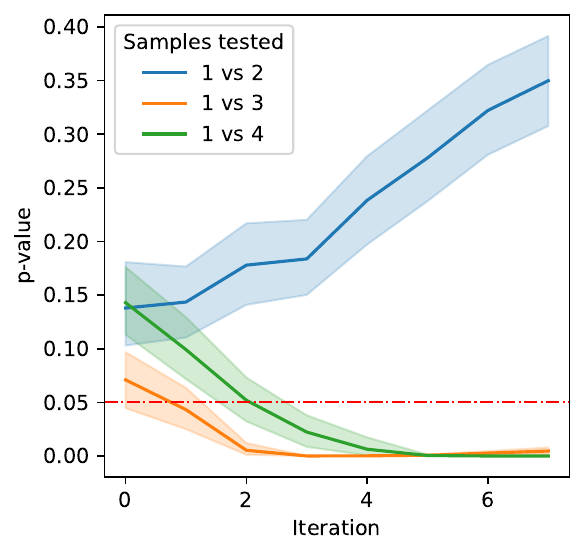}
\caption{Homogeneity testing of various configurations of mixtures}
\label{fig:homo_test}
\end{figure}


\end{document}